\PassOptionsToPackage{sort&compress,round}{natbib}
\documentclass{article}
\usepackage{iclr2027_conference,times}

\usepackage[utf8]{inputenc}
\usepackage[T1]{fontenc}
\usepackage{graphicx}
\graphicspath{{figures/}}
\usepackage[export]{adjustbox}
\usepackage{hyperref}
\hypersetup{colorlinks=false,pdfborder={0 0 0},pdfcreator={},pdfproducer={}}
\pdftrailerid{}
\usepackage{xcolor}
\usepackage{url}
\usepackage{amsfonts}
\usepackage{amsmath}
\usepackage{amssymb}
\usepackage{amsthm}
\usepackage{microtype}
\usepackage{bm}
\usepackage{array}
\usepackage{booktabs}
\usepackage{enumitem}
\usepackage{placeins}
\usepackage[capitalize,noabbrev]{cleveref}
\usepackage{tikz}
\usetikzlibrary{positioning, arrows.meta, calc, shapes.geometric, fit, backgrounds}
\definecolor{lifthyper}{HTML}{ff7f0e}
\definecolor{liftdirect}{HTML}{d62728}
\definecolor{liftpgd}{HTML}{2ca02c}
\definecolor{liftshared}{HTML}{4a4a4a}
\newcommand{\lift}{\textcolor{lifthyper}{lift}}
\newcommand{\liftdef}{\textcolor{lifthyper}{lift}}

\newcommand{\direct}{\textcolor{liftdirect}{direct softplus}}
\newcommand{\directdef}{\textcolor{liftdirect}{direct softplus}}
\newcommand{\directcap}{\textcolor{liftdirect}{Direct softplus}}
\newcommand{\pgd}{\textcolor{liftpgd}{PGD}}
\newcommand{\pgddef}{\textcolor{liftpgd}{PGD}}

\theoremstyle{plain}
\newtheorem{proposition}{Proposition}
\newtheorem{lemma}{Lemma}
\newtheorem{theorem}{Theorem}
\newtheorem{corollary}{Corollary}
\theoremstyle{definition}
\newtheorem{remark}{Remark}
\newtheoremstyle{assump}{3pt}{3pt}{\itshape}{}{\bfseries}{.}{ }{}
\theoremstyle{assump}
\newtheorem{assumption}{Assumption}

\renewcommand{\theconvention}{C\the\numexpr\value{convention}-1\relax}
\theoremstyle{plain}
\crefname{assumption}{Assumption}{Assumptions}
\crefname{assumptionr}{Assumption}{Assumptions}
\crefname{convention}{Convention}{Conventions}

\newcommand{\R}{\mathbb{R}}
\newcommand{\E}{\mathbb{E}}
\newcommand{\Var}{\mathrm{Var}}
\newcommand{\Cov}{\mathrm{Cov}}
\newcommand{\tr}{\mathrm{tr}}

\newcommand{\bb}{{\bm{b}}}
\newcommand{\be}{{\bm{e}}}

\newcommand{\bg}{{\bm{g}}}
\newcommand{\bH}{{\bm{H}}}
\newcommand{\bJ}{{\bm{J}}}
\newcommand{\br}{{\bm{r}}}
\newcommand{\bv}{{\bm{v}}}
\newcommand{\bV}{{\bm{V}}}
\newcommand{\bxi}{{\bm{\xi}}}
\newcommand{\bSigma}{{\bm{\Sigma}}}
\newcommand{\bP}{{\bm{P}}}

\newcommand{\btheta}{{\bm{\theta}}}
\newcommand{\bphi}{{\bm{\phi}}}
\newcommand{\bu}{{\bm{u}}}
\newcommand{\bX}{{\bm{X}}}

\newcommand{\bx}{{\bm{x}}}

\renewcommand{\L}{\mathcal{L}}

\iclrfinalcopy
\title{A lift for input-convex neural net training}
\author{Ali Siahkoohi \\
Department of Computer Science \\
University of Central Florida \\
\texttt{alisk@ucf.edu}}
\date{}

\begin{document}
\maketitle

\begin{abstract}
Input-convex neural nets parametrize the convex potentials of density models and transport maps, and their convexity requires the inter-layer weights to be non-negative. Projected gradient descent enforces this by projecting after each step, and due to mini-batch noise the boundary is re-crossed indefinitely, which leads to an active set the projection never identifies. The differentiable alternative, direct softplus, optimizes a free latent weight through a softplus positivity map. Where that latent weight is well below zero, the derivative of the map is exponentially small and scales down both the gradient and its noise. We call this region the shoulder, and a weight that reaches it therefore barely moves and stays for an exponentially long time. We propose the lift, which replaces the free latent weight by a learnable slack plus the output of a hypernetwork that reads a permutation-invariant summary of the training batch. The latent weight then varies with the same batch that forms the gradient, so the two are coupled. We show that this coupling adds a term to the variance of the update at first order, and that the slack, the batch dependence, and the shared batch are each needed for it. Where the coupling aligns positively with the loss curvature, a weight on the shoulder crosses the barriers of the loss sooner under the lift than under direct softplus, also under preconditioned stochastic gradient descent. We compare the lift with the two existing methods on several applications. Where a constrained weight of direct softplus reaches the shoulder, the lift fits the target more closely than direct softplus and reaches the same reconstruction about three times sooner. Where almost none reaches it, the methods agree.
\end{abstract}
 \section{Introduction}
\label{sec:intro}

Input-convex neural nets (ICNNs)~\citep{amos2017icnn} are feed-forward
networks whose output is convex in the input. They parametrize the
negative log-density of a log-concave energy-based
model~\citep{lin2023falcon,vesseron2025moment}, the potential of a
convex potential flow~\citep{huang2021cpflow,kan2022mqf}, the transport maps that sample the posterior of a Bayesian inverse
problem~\citep{taghvaei2022bayes,wang2024conditional}, the Brenier potential
of neural optimal
transport~\citep{makkuva2020icnnot,korotin2021benchmark,bunne2022supervised,amos2023amortizing},
and the convex
regularizers of learned image
reconstruction~\citep{cui2024unfolding,mukherjee2020learned,ehrhardt2026primaldual}.
These uses share one parametrization, which requires convex non-decreasing
activations and non-negative inter-layer weights
$\btheta\ge\bm{0}$~\citep{amos2017icnn}.

Two methods enforce this constraint in practice.
\textcolor{liftpgd}{Projected gradient descent}
(\pgddef{})~\citep{amos2017icnn} returns any negative weight to zero
by the projection $\max(\btheta,\bm{0})$ after each step. While a coordinate
held at zero receives the full gradient and leaves as soon as that update
points inward, mini-batch noise sends it back. This
noise does not vanish as the mean gradient does, so the boundary is re-crossed
indefinitely and the projection never identifies the coordinates at
zero~\citep{duchi2021asymptotic}.
The differentiable alternative, \direct{}, optimizes a free latent weight
through the softplus positivity map, whose derivative is exponentially small
where that weight is negative. On the shoulder, where this derivative is small, a weight barely moves, and
the time it takes to cross a barrier of the loss is exponential in the barrier
height over the attenuated gradient variance~\citep{kramers1940brownian}. The two methods therefore sit at
opposite extremes in how freely a coordinate near zero keeps moving once
training has taken it there.

To let a weight that reaches the shoulder cross the barriers of the loss sooner,
we propose the \lift{}, which keeps the differentiable unconstrained parametrization of \direct{}, changing only what feeds the positivity map.
Borrowing the extension of PDE-constrained inverse problems~\citep{vanleeuwen2013,symes2020source}, we replace the free latent
weight by a learnable slack plus the output of an unconstrained
hypernetwork~\citep{ha2017hypernetworks} that reads a permutation-invariant
summary~\citep{zaheer2017deepsets} of the training batch (\cref{fig:lift}).
Since the batch is resampled at each iteration, the latent weight fluctuates with it,
and that fluctuation couples to the gradient on the same batch. We show that the coupling adds a term to the
variance of the update to the latent weight at first order in the
fluctuation (\cref{thm:main}). The slack, the batch dependence, and the shared batch are
each necessary for this coupling (\cref{thm:joint-necessity} and the
substitutions that follow it). Where the
coupling term plus the remainder is positive, the variance is larger than under \direct{}, which lowers the exponent of the time a weight on the shoulder takes to cross a
barrier of the loss (\cref{thm:main} and \cref{cor:fpt}). This ordering holds
as well under preconditioned stochastic gradient descent, e.g., Adam without momentum, where the separation keeps its sign and lies between its value under plain steps and the square root of that value (\cref{prop:precond}).
\begin{figure}[t]
\centering
\begin{tikzpicture}[
  >={Stealth[length=2.4mm]},
  every node/.append style={font=\small},
  hyperbox/.style={draw=lifthyper, fill=lifthyper!10, line width=1.0pt, rounded corners=2pt, inner sep=4pt, align=center},
  sharedbox/.style={draw=liftshared, fill=liftshared!8, line width=0.8pt, rounded corners=2pt, inner sep=4pt, align=center},
  slackarrow/.style={->, line width=1.6pt, draw=lifthyper},
  bodyarrow/.style={->, line width=1.0pt, draw=lifthyper},
  sharedarrow/.style={->, line width=1.0pt, draw=liftshared!80},
  hbox/.style={hyperbox, minimum width=1.45cm, minimum height=0.95cm},
]
\node[hyperbox, fill=lifthyper!4, anchor=west] (batch) at (0, 0) {
  $\bx_1$\\[-1pt]$\bx_2$\\[-1pt]$\vdots$\\[-1pt]$\bx_n$
};
\node[hbox, anchor=west] (h1)
  at ($(batch.east) + (0.42cm, 0)$) {$h_\bphi^{(1)}$\\\scriptsize per-point};
\node[circle, draw=lifthyper, fill=lifthyper!12, line width=1.0pt, minimum size=0.9cm, inner sep=0pt, font=\small, anchor=west]
  (pool) at ($(h1.east) + (0.36cm, 0)$) {$\tfrac{1}{n}\!\sum$};
\node[hbox, anchor=west] (h2)
  at ($(pool.east) + (0.36cm, 0)$) {$h_\bphi^{(2)}$\\\scriptsize affine};
\node[circle, draw=lifthyper, fill=white, line width=1.2pt, minimum size=0.8cm, inner sep=0pt, font=\large, anchor=west]
  (sum) at ($(h2.east) + (0.42cm, 0)$) {$\oplus$};
\node[hbox, anchor=south] (slack) at ($(h2.north) + (0, 0.45cm)$) {$\bb$\\\scriptsize slack};
\path let \p1 = (slack.north), \p2 = (batch.south) in
  \pgfextra{\pgfmathsetlengthmacro{\sharedH}{\y1-\y2}\global\let\sharedH\sharedH};
\coordinate (tail-sw) at (sum.east |- batch.south);
\node[sharedbox, minimum height=\sharedH, minimum width=1.0cm, anchor=south west]
  (psi) at ($(tail-sw) + (1.0cm, 0)$) {$\psi(\cdot)$\\\scriptsize softplus};
\node[sharedbox, minimum height=\sharedH, minimum width=1.0cm, anchor=south west]
  (theta) at ($(psi.south east) + (0.45cm, 0)$) {$\btheta\ge\bm{0}$\\\scriptsize constrained};
\node[sharedbox, minimum height=\sharedH, minimum width=1.5cm, anchor=south west]
  (icnn) at ($(theta.south east) + (0.45cm, 0)$) {ICNN\\$E_\btheta(\bx)$};
\draw[bodyarrow] (batch.east) -- (h1.west);
\draw[bodyarrow] (h1.east) -- (pool.west);
\draw[bodyarrow] (pool.east) -- (h2.west);
\draw[bodyarrow] (h2.east) -- (sum.west);
\draw[slackarrow, rounded corners=4pt] (slack.east) -| (sum.north);
\draw[bodyarrow] (sum.east) -- (psi.west |- sum.center);
\draw[sharedarrow] (psi.east) -- (theta.west);
\draw[sharedarrow] (theta.east) -- (icnn.west);
\end{tikzpicture}
\caption{The \lift{} of \eqref{eq:lift}: a permutation-invariant hypernetwork reads
the batch, its output plus a slack passes through the softplus positivity
map, and the result is the constrained weight of the ICNN.}
\label{fig:lift}
\end{figure}
 
We evaluate the \lift{} against \direct{} on energy-based models, convex
potential flows, a transport map for a high-dimensional posterior, and a
convex regularizer for image reconstruction. We also run \pgd{} on the energy-based models and on the transport map and
the image regularizer, since the published work behind them uses it. The \lift{} is closer to every one-dimensional target than \direct{} on every
paired run, and below \pgd{} on every seed of three of the four. When a constrained weight of \direct{} reaches the shoulder, the \lift{} fits
the target more closely than \direct{} on every energy-based run that returns a likelihood, and
reaches the same reconstruction about three times sooner on the image
regularizer. When almost no weight does, the methods agree. Initialization schemes for ICNNs set
the starting values of the constrained
weights~\citep{hoedt2023principled,hundrieser2026hycnn}, not how a weight moves once it is near zero. We run the transport map from three initializations, and on the image
regularizer \direct{} does not train at all from the one
\citet{hoedt2023principled} prescribe. What decides whether the methods differ
is not the application but how freely a weight that has reached the shoulder still moves there.

The paper is organized as follows. We state the \lift{} in \cref{sec:method},
derive the escape time it changes in \cref{sec:diffusion}, report the
results in \cref{sec:experiments}, and place it among the alternatives in
\cref{sec:related}.

\section{The shoulder and the \texorpdfstring{\textcolor{lifthyper}{lift}}{lift}}
\label{sec:method}

The constraint $\btheta\ge\bm{0}$ on the inter-layer weights of an
ICNN~\citep{amos2017icnn} is the sole obstacle to training one with standard
gradient methods. We write $\btheta\in\R^d$ for these weights and leave the network
$E_{\btheta}$'s unconstrained parameters implicit.
Our running example is a log-concave energy-based model (EBM)
$p(\bx)\propto\exp(-E_{\btheta}(\bx))$ trained by a stochastic gradient
method on a loss $\L$ over mini-batches $\bX=\{\bx_1,\ldots,\bx_n\}$. 

\subsection{The shoulder of the positivity map}
\label{sec:shoulder}

The differentiable method applies a
monotone positivity map $\btheta=\psi(\tilde\btheta)$, with
$\psi:\R\to[0,\infty)$ acting coordinate-wise, and optimizes the free
latent weight $\tilde\btheta\in\R^d$ in its place. Monotonicity matters: a
non-monotone map with the same infimum, such as the square~\citep{berthier2023incremental}, can introduce spurious critical
points into the landscape~\citep{levin2025effect}. The standard
choice is the softplus, $\psi(u)=\log(1+e^{u})$, whose derivative is the
logistic sigmoid $\psi'(u)=(1+e^{-u})^{-1}\in(0,1)$, and we call this
method \directdef{}. By the chain rule the gradient that reaches the latent
weight is the one with respect to the constrained weight, multiplied
elementwise by that derivative,
\begin{equation}\label{eq:chain}
\bg \;=\; \nabla_{\tilde\btheta}\L \;=\; \psi'(\tilde\btheta)\odot\nabla_{\btheta}\L,
\end{equation}
and $\psi'(u)$ decays as $e^{u}$ for negative $u$. We call the region in which
$\psi'(\tilde\btheta_j)\le\sigma_\mathrm{sh}$ for a threshold
$\sigma_\mathrm{sh}\ll1$ the \emph{shoulder}, with
$\sigma_\mathrm{sh}=0.05$ throughout, which puts the shoulder at latent
weights below $-2.94$. A coordinate there has the mean and the mini-batch noise of its gradient scaled by
that derivative, so it barely moves, and at most a fraction $\sigma_\mathrm{sh}^2$ of the gradient
variance reaches the latent weight. The time to cross a barrier is exponential in the ratio of the
barrier height to that variance, up to the step
size~\citep{kramers1940brownian,hanggi1990reaction}, which \cref{thm:main} states for this
setting and \cref{prop:precond} for preconditioned steps.  Nothing in the parametrization keeps a
constrained weight off the shoulder,
either at initialization or over the course of
training. The exponential map behaves the same there (\cref{app:expmap}).

The projected alternative keeps the raw weights and follows each gradient
step by the coordinate-wise projection $\max(\cdot,\bm 0)$. Since \pgd{}
does not differentiate through this projection, each raw coordinate receives
$\nabla_{\btheta}\L$ undiminished and leaves the boundary as soon as its
update points inward. The length of a visit to the boundary therefore follows
the memoryless law of the projected recursion in \cref{prop:clamp}. We compare \pgd{} with the \lift{} empirically.
 
\subsection{The \texorpdfstring{\textcolor{lifthyper}{lift}}{lift}}
\label{sec:lift}
\label{sec:lift-def}

The \liftdef{} replaces the free latent weight by the sum of a learnable
vector $\bb\in\R^d$, the \emph{slack}, and the output of a
permutation-invariant hypernetwork $h_\bphi$ that reads the mini-batch,
\begin{equation}\label{eq:lift}
\btheta=\psi(\tilde\btheta),\qquad
\tilde\btheta=\bb+h_\bphi(\bX).
\end{equation}
On the energy-based models and the image regularizer the hypernetwork emits the
unconstrained parameters as well, which need no positivity map. Invariance to the ordering of
the batch follows from pooling the per-point
contributions~\citep{zaheer2017deepsets}. Training minimizes the loss over the lifted parameters,
\begin{equation}\label{eq:lift-objective}
\min_{\bphi,\,\bb}\;\L\bigl(\psi(\bb+h_\bphi(\bX))\bigr),
\end{equation}
and fixing $h_\bphi\equiv\bm{0}$ recovers \direct{}, so that the slack is
the free latent weight defined in \cref{sec:shoulder} (\cref{fig:lift}). The hypernetwork reads the batch on which the gradient
of $\L$ is formed at the same iteration. Its dependence on the batch acts during training alone. At each evaluation the
hypernetwork reads a
reference set of samples once. It is the training set on the tabular
energy-based models, the flows and the transport map, a fixed sample of the target on the
one-dimensional energy-based models, and on the image regularizer a fixed batch for validation and
the training set for the test. On the energy-based models the held-out loss that selects the reported iterate reads a fresh batch instead. The ICNN is then a fixed network
with \direct{}'s architecture and parameter count.

\section{Diffusion and escape}
\label{sec:diffusion}
\label{sec:ingredients}
\label{sec:curvature}

Since the time a weight takes to cross a barrier of the loss is exponential in
the barrier height over the variance of its update, we compute that variance
first. \Cref{thm:main}\textup{(i)} writes it as its \direct{} value plus three
terms, a coupling term between the fluctuation of the latent weight and the
gradient noise of the same batch, a fluctuation term, and a remainder. Only covariances with that gradient enter at
first order in the fluctuation, so a fluctuation independent of the batch first
enters at second order (\cref{thm:main}\textup{(ii)}). Wherever the coupling term plus the remainder is positive, a weight on the shoulder crosses a barrier of the loss sooner under the \lift{}
at small step sizes (\cref{thm:main}\textup{(iii)} and \cref{cor:fpt}). \Cref{thm:joint-necessity} and the
substitutions after it show which parts of the \lift{} the coupling requires,
and \cref{prop:precond} extends the ordering to preconditioned stochastic
gradient descent. The results of this section are proved in \cref{app:proofs} and machine-checked in
Lean~4 against mathlib.

\subsection{The fluctuation and its coupling}
\label{sec:mechanism}
At fixed $\bphi$ and $\bb$ the latent weight of \eqref{eq:lift} varies with the
batch by
\begin{equation}\label{eq:fluct}
\delta\tilde\btheta \;=\; h_\bphi(\bX)-\E_{\bX}\bigl[h_\bphi(\bX)\bigr],
\end{equation}
the \emph{fluctuation} the \lift{} adds to the training dynamics. While \direct{}
trains that weight as a parameter no batch changes, the \lift{} redraws it at each step.

To see what the fluctuation does to the update, we expand the gradient of a
batch about the mean latent weight $\bar{\tilde\btheta}=\bb+\E_\bX[h_\bphi(\bX)]$,
at which both methods are read. With
$\bg_0(\bX)=\nabla_{\tilde\btheta}\L(\bar{\tilde\btheta};\bX)$ the gradient of a
batch there,
\[
\bxi=\bg_0-\E_\bX[\bg_0],\quad
\bH_{\tilde\btheta}=\E_\bX[\nabla^2_{\tilde\btheta}\L(\bar{\tilde\btheta};\bX)],\quad
\bV=\Cov_\bX[\delta\tilde\btheta],\quad
\bSigma_\xi=\E_\bX[\delta\tilde\btheta\,\bxi^\top],
\]
the gradient of a batch at its own latent weight is
$\bg=\bg_0+\bH_{\tilde\btheta}\,\delta\tilde\btheta+\br_2$, which defines the
remainder $\br_2$, and we call $\bSigma_\xi$, the covariance of the fluctuation
with the gradient noise of the same batch, the \emph{coupling}. It vanishes for
\pgd{}, whose constrained weights are stored parameters (\cref{thm:pgd-zero}). Throughout,
$\bg_0$, $\delta\tilde\btheta$, $\br_2$ and the Hessian of a batch are square
integrable and $\bH_{\tilde\btheta}$ is symmetric.

Since the shoulder is defined coordinate by coordinate, we read the update along a
coordinate direction $\be_b$ of the latent weight,
which is also a coordinate direction of the slack, since
$\partial\tilde\btheta/\partial\bb$ is the identity. Along it the variance of
the update and the gradient noise of the constrained weight are
\[
\sigma^2_\mathrm{eff}=\be_b^\top\Cov_\bX[\bg]\,\be_b,\qquad
\sigma^2_\mathrm{obj}=\be_b^\top\Var_\bX[\nabla_\btheta\L]\,\be_b,
\]
where we assume $\sigma^2_\mathrm{obj}>0$, and the remainder enters through
$\varrho=\sigma^2_\mathrm{eff}-\be_b^\top\Cov_\bX[\bg_0+\bH_{\tilde\btheta}\delta\tilde\btheta]\,\be_b$,
whose expansion in $\br_2$ is in \cref{app:proofs}. What remains is how the coupling changes
$\sigma^2_\mathrm{eff}$, and how $\sigma^2_\mathrm{eff}$ sets the time to cross a
barrier.

\subsection{Hypotheses}
\label{sec:hypotheses}
Parts \textup{(i)} and \textup{(ii)} of \cref{thm:main} require only the
regularity stated in \cref{sec:mechanism} and a positive derivative of the positivity
map at the mean latent weight, while part \textup{(iii)} requires a diffusion and a barrier.

\begin{assumption}[Independent batches and an It\^o process]\label{as:A1}
The lifted parameters $\bm{\varphi}=(\bb,\bphi)$ of \eqref{eq:lift-objective}
are trained by stochastic gradient descent with step size $\eta>0$ on
mini-batches $\bX_1,\bX_2,\ldots$ that are independent and identically
distributed, and the recursion is read as the It\^o process
$d\bm{\varphi}=-\nabla_{\bm{\varphi}}\tilde\L(\bm{\varphi})\,dt+\sqrt{\eta}\,\bSigma^{1/2}(\bm{\varphi})\,d\bm{B}_t$,
with $\tilde\L(\bm{\varphi})=\E_{\bX}[\L(\psi(\bb+h_\bphi(\bX)))]$ the expected
loss and $\bSigma(\bm{\varphi})$ the covariance of the per-step update, whose
entries are Lipschitz in $\bm{\varphi}$ on a region containing the interval
$[a,b]$ of Assumption~\ref{as:A2}, on which the drift
$\nabla_{\bm{\varphi}}\tilde\L$ is continuous.
\end{assumption}
We take the continuous-time process in place of the recursion, following the
stochastic modified equations of \citet{li2017sde}, and read both methods under
one process, one batch law and one expected loss. The two exponents then
share the barrier height and the same modeling error. Parts \textup{(i)} and
\textup{(ii)} hold without it.

\begin{assumption}[Slope and barrier]\label{as:A2}
\textup{(a)} The positivity map $\psi$ is continuously differentiable and
non-decreasing, and on an interval $\mathcal{B}\subset\R$ its derivative lies
within a tolerance $\varepsilon$ of one constant, $|\psi'(w)-\sigma_s|\le\varepsilon$
for every $w\in\mathcal{B}$, with $\sigma_s=\psi'(\tilde w_s)>0$ at a point
$\tilde w_s$ of $\mathcal{B}$ and $0\le\varepsilon<\sigma_s$.
\textup{(b)} Along $\be_b$ the expected loss $\tilde\L$ is continuous on an
interval $[a,b]$ and differentiable in its interior, and there are
$a<w_b<w_s<b$ such that $w_s$ maximizes and $w_b$ minimizes $\tilde\L$ on
$[a,b]$, with $\tilde\L(w_b)<\tilde\L(w_s)$ and $\tilde\L$ non-decreasing on
$[w_b,w_s]$.
\end{assumption}
For softplus part \textup{(a)} holds automatically (\cref{app:assumptions}), so
its content is that the weights sit in such an interval, one on which $\sigma_s$ falls below the shoulder threshold of \cref{sec:shoulder}.

\subsection{Escape time}
\label{sec:escape}
\Cref{thm:main} computes the diffusion of the \lift{} exactly and inserts it,
with the \direct{} value, into the classical law for the time to cross a
barrier~\citep{kramers1940brownian}.

\begin{theorem}[The coupling term and the escape time]\label{thm:main}
Fix $\bar{\tilde\btheta}$ and the coordinate direction $\be_b$ of
\cref{sec:mechanism}, with $\be_b^\top\bar{\tilde\btheta}$ in the interval
$\mathcal{B}$ of Assumption~\ref{as:A2}\textup{(a)} taken as its point $\tilde w_s$, so that $\sigma_s=\psi'(\be_b^\top\bar{\tilde\btheta})$. Parts \textup{(i)} and \textup{(ii)} use the
regularity of \cref{sec:mechanism} and Assumption~\ref{as:A2}\textup{(a)}, and part \textup{(iii)} adds
Assumption~\ref{as:A1} and Assumption~\ref{as:A2}\textup{(b)}. Then the
following hold.

\textup{(i)} \emph{Decomposition.} With no term discarded,
\begin{equation}\label{eq:sigma-eff}
\sigma^2_\mathrm{eff} \;=\;
\underbrace{\sigma_s^2\sigma^2_\mathrm{obj}}_{\text{\direct{} value}}
\;+\; \underbrace{2\,\be_b^{\top}\bH_{\tilde\btheta}\bSigma_\xi\,\be_b}_{\text{coupling term}}
\;+\; \underbrace{\be_b^{\top}\bH_{\tilde\btheta}\bV\bH_{\tilde\btheta}\,\be_b}_{\text{fluctuation term}}
\;+\; \underbrace{\varrho}_{\text{remainder}},
\end{equation}
the \direct{} value at $\bar{\tilde\btheta}$ being the first term alone, within
the bracket of Assumption~\ref{as:A2}\textup{(a)}. The fluctuation term is
$\Var_\bX[\langle\bH_{\tilde\btheta}\be_b,\delta\tilde\btheta\rangle]\ge0$, so
$\sigma^2_\mathrm{eff}-\sigma_s^2\sigma^2_\mathrm{obj}\ge c$ for any coupling,
with $c=2\,\be_b^{\top}\bH_{\tilde\btheta}\bSigma_\xi\be_b+\varrho$ the coupling
term plus the remainder.

\textup{(ii)} \emph{Orders.} Scale the fluctuation and the
remainder at a fixed batch law, $\delta\tilde\btheta=t\,\bu$ and
$\br_2=t\,\bm{w}_1+t^2\bm{w}_2$. Then
$\sigma^2_\mathrm{eff}(t)-\sigma_s^2\sigma^2_\mathrm{obj}$ is a polynomial in
$t$ with no constant term, whose linear coefficient pairs $\bg_0$ with $\bu$
through $\bH_{\tilde\btheta}$ and with $\bm{w}_1$ directly, and whose quadratic
coefficient contains the fluctuation term
$\be_b^\top\bH_{\tilde\btheta}\Cov_\bX[\bu]\bH_{\tilde\btheta}\be_b$; both in
full in \cref{lem:diffusion}\textup{(vi)}. A latent weight independent of the
batch has $\bSigma_\xi=\bm 0$ and $\Cov_\bX[\bg_0,\bm w_1]=\bm 0$ as well, $\bm w_1$
being the curvature-fluctuation part of the remainder, so its excess begins at second order.
A positive linear coefficient gives
$\sigma^2_\mathrm{eff}>\sigma_s^2\sigma^2_\mathrm{obj}$ at every sufficiently
small $t>0$, and a negative one reverses it.

\textup{(iii)} \emph{Escape, the classical law at the two diffusions.} Let
$\alpha=\tilde\L(w_s)-\tilde\L(w_b)>0$ be the barrier height along $\be_b$, read $\sigma^2_\mathrm{eff}$ as constant across the barrier, at the \direct{} value
$\sigma_s^2\sigma^2_\mathrm{obj}$ for \direct{} and at \eqref{eq:sigma-eff} for
the \lift{}, and for the diffusion
$dw=-\tilde\L'(w)\,dt+\sqrt{\eta}\,\sigma_\mathrm{eff}\,dB_t$ from a start
$x<w_s$ on the interval $[a,b]$ of Assumption~\ref{as:A2}\textup{(b)},
reflecting at $a$ and absorbing at $b$, let $\tau$ be the first passage to $b$,
$\mathsf{T}(x)=\E_x[\tau]$ its mean, and
$\mathsf{D}(x)=\E_x[\int_0^{\tau}\mathbf 1_{[a,w_s]}(w_t)\,dt]$ the mean time
spent below the barrier top before it, both under Dynkin's formula for its
generator at every step size and continuous up to the closed interval, with
$\mathsf{D}$ continuously differentiable on $(a,b)$. Then
$\eta\sigma^2_\mathrm{eff}\log\mathsf{T}(x)\to2\alpha$ and
$\eta\sigma^2_\mathrm{eff}\log\mathsf{D}(x)\to2\alpha$, while
$\eta\sigma^2_\mathrm{eff}\log\bigl(\mathsf{T}(x)/\mathsf{D}(x)\bigr)\to0$, as
$\eta\sigma^2_\mathrm{eff}\to0^{+}$.
\end{theorem}

The exponent $2\alpha/(\eta\sigma^2_\mathrm{eff})$ is strictly decreasing in
$\sigma^2_\mathrm{eff}$, so the method with the larger effective diffusion
has the smaller exponent and, for every sufficiently small step size, the
shorter mean first passage and the shorter mean time below the barrier top
(\cref{cor:fpt,cor:dwell}). Under $c>0$ part \textup{(i)} puts the \lift{} on
that side; the \lift{} enters part \textup{(iii)} through
$\sigma^2_\mathrm{eff}$ alone. The ordering rests on $c>0$, a sign the method does not fix.

\directcap{}'s exponent in part \textup{(iii)} is of order $\sigma_s^{-2}$
within the shoulder
under no further condition, and the \lift{}'s is too wherever the remainder is itself
of order $\sigma_s^2$ with a constant below the leading coefficient of
\eqref{eq:sigma-eff}. Under $c>0$ the \lift{} therefore lowers their constant
and not their order, by the factor
$\sigma^2_\mathrm{eff}/(\sigma_s^2\sigma^2_\mathrm{obj})$ of part \textup{(i)}
(\cref{rem:power}). The ordering is stated at constant diffusion, and where $\sigma^2_\mathrm{eff}$
varies across the barrier, \cref{lem:diffusion}\textup{(v)} bounds the effective
action of each method, which replaces $\alpha/\sigma^2_\mathrm{eff}$ in its stationary density (\cref{lem:diffusion}\textup{(iv)}), between its values at the largest and the smallest
diffusion, and we do not show that the two brackets separate.

\subsection{Joint necessity}
\label{sec:necessity}
\Cref{thm:joint-necessity} and the substitutions after it state
which parts of the method the coupling and its estimate need. With
$\tilde\btheta_k=\bb+h_\bphi(\bX_k)$ the latent weight of the $k$-th batch,
$\bg_k$ the gradient of \eqref{eq:chain} at $\tilde\btheta_k$ on that batch,
and $\bar{\tilde\btheta}$, $\bar\bg$ their sample means,
\begin{equation}\label{eq:cross-cov}
\widehat\bSigma^{(t)} \;=\; \frac{1}{T-1}\sum_{k=1}^{T}
\delta\tilde\btheta_k\,\delta\bg_k^{\top},
\qquad
\delta\tilde\btheta_k=\tilde\btheta_k-\bar{\tilde\btheta},\quad
\delta\bg_k=\bg_k-\bar\bg.
\end{equation}

\begin{theorem}[Joint necessity]\label{thm:joint-necessity}
Fix an iterate, so that $\bb$ and $\bphi$ are held, and let $\widehat\bSigma^{(t)}$ be the estimator of
\eqref{eq:cross-cov} formed on $T$ batches sampled at that iterate, with either
divisor $T-1$ or $T$. If the latent weight $\tilde\btheta$ and the gradient
$\bg_0$ at the reference iterate $\bb+\E_\bX[h_\bphi(\bX)]$ have finite fourth moments and are
conditionally independent given the parameters, as they are when the
hypernetwork reads a batch drawn independently of the one that forms the
gradient, then $\bSigma_\xi=\bm 0$ for almost every value of the parameters, and
under Assumption~\ref{as:A1} the estimate $\widehat\bSigma_\xi^{(t)}$ formed
with $\bg_0(\bX_k)$ in place of $\bg_k$ has mean zero for every $T\ge2$, has an
entrywise mean-square error of order $1/T$, and converges to zero almost surely
as $T\to\infty$. A nonzero coupling therefore requires that the latent weight
and the gradient at the reference
iterate not be conditionally independent.
\end{theorem}

The other two parts zero the estimate by substitution. With $\bJ_\bb=\partial\tilde\btheta/\partial\bb$ the slack Jacobian, a latent weight
with no learnable additive slack has $\bJ_\bb\equiv\bm 0$ and
$\bJ_\bb^\top\widehat\bSigma^{(t)}=\bm 0$, and a
hypernetwork constant in the conditioning batch has $\delta\tilde\btheta_k\equiv\bm 0$, so
every summand of \eqref{eq:cross-cov} vanishes; both hold identically, at every
iterate, for every $T\ge2$ at divisor $T-1$ and every $T\ge1$ at divisor $T$. Removing any one of the three parts therefore sets
the coupling in the update of the slack, $\bJ_\bb^\top\bSigma_\xi$, to zero, by \cref{thm:joint-necessity} and by the two
substitutions above.

\subsection{Preconditioned steps}
\label{sec:precond}
Every run in \cref{sec:experiments} trains with Adam, which divides the step of
each parameter by the square root of a running estimate of its second moment,
while Assumption~\ref{as:A1} reads the recursion as stochastic gradient descent.
We therefore ask what of \cref{thm:main} survives a preconditioned step.

\begin{proposition}[Preconditioned steps]\label{prop:precond}
Let Assumption~\ref{as:A1} hold with each step of stochastic gradient descent
multiplied by a diagonal matrix $\bP$ with positive entries that depends only on
the batches sampled before that step, and read its entry $p$ along $\be_b$,
$\bP\be_b=p\,\be_b$, as constant across the barrier, as
$\sigma^2_\mathrm{eff}$ is. Then \cref{thm:main}\textup{(i)}--\textup{(ii)} and
\cref{thm:joint-necessity} hold unchanged, and \cref{thm:main}\textup{(iii)}
holds with $\eta\,p\,\sigma^2_\mathrm{eff}$ in place of $\eta\sigma^2_\mathrm{eff}$.
If moreover $p=\chi(m^2+\sigma^2_\mathrm{eff})$, with $m$ the mean gradient along
$\be_b$ at $\bar{\tilde\btheta}$, shared by the two methods, and $\chi>0$ non-increasing such that $v\mapsto v\,\chi(m^2+v)$ is
strictly increasing, then the exponent $2\alpha/(\eta\,p\,\sigma^2_\mathrm{eff})$
is strictly decreasing in $\sigma^2_\mathrm{eff}$. Whenever the factor
$\kappa=\sigma^2_\mathrm{eff}/(\sigma_s^2\sigma^2_\mathrm{obj})$ of
\cref{thm:main}\textup{(i)} exceeds one, the ratio of the \direct{} exponent to
the \lift{}'s exceeds one and is at most $\kappa$, and under Adam's
normalization, $\chi(v)=(v^{1/2}+\epsilon)^{-1}$ with $\epsilon\ge0$ its stabilizing constant, it is at least $\kappa^{1/2}$.
\end{proposition}

Plain stochastic gradient descent is the case $\chi\equiv1$, where the ratio is
$\kappa$ itself. Under $c>0$ the \lift{} therefore keeps the shorter mean first
passage and the shorter time below the barrier top at small step sizes for any such
preconditioner, Adam without momentum included. Within the shoulder Adam's normalization also lowers the order of both
exponents alike, from $\sigma_s^{-2}$ to $\sigma_s^{-1}$ at $\epsilon=0$, the \lift{}'s exponent under the same condition on its remainder, so the coupling still lowers their
constant and not their order (\cref{rem:power-precond}). The proof is in
\cref{app:precond}. Adam's averaging of the first moment lies outside the
statement, and so does the share of the current batch in its second-moment
estimate, which the statement admits only lagged by one step. Every separation reported in \cref{sec:experiments} is measured under Adam, where,
up to these two departures, the ratio of the exponents is at most $\kappa$.

\section{Experiments}
\label{sec:experiments}

Two questions run through the experiments: whether training enters the
shoulder defined in \cref{sec:shoulder}, and what the three methods do
where it does. We therefore report the share of constrained weights on the shoulder for all problems. We compare the methods on the
training objective they minimize, and read the total-variation distance to the
target, the held-out negative log-likelihood, or the reconstruction PSNR at the best held-out read, except where we note otherwise. We pair the methods run for run, i.e., they share the seed and the
budget. Training details are in \cref{app:details} and the held-out read along training in \cref{app:val-loss}.

\subsection{Log-concave energy-based models}
\label{sec:exp-ebm}\label{sec:exp-pgd}

\paragraph{One-dimensional log-concave targets.}
We train our running example (\cref{sec:method}) by the forward
Kullback--Leibler divergence on the Gumbel, Laplace, Gamma, and Beta targets
in one dimension, at thirty paired seeds. Here \direct{} begins with a third of its constrained weights on the shoulder
and ends with more on three of the four targets. The \lift{} sits closer to
the target in total-variation distance than \direct{} on all one hundred twenty
paired runs, and its worst seed on each target is below \direct{}'s best
(\cref{fig:densities}; statistics in \cref{app:gallery}). Against \pgd{} the \lift{} is lower on every seed than \pgd{} is on any seed
of Gumbel, Laplace, and Beta, while on Gamma it leads at the median and the
two seed bands overlap. The \lift{} also reaches the lowest median training loss on all four targets,
and on each of them the separation opens within the first few hundred
iterations. \directcap{} never closes it, and \pgd{} closes it only on Gamma.

\begin{figure}[t]
\centering
\includegraphics[width=\linewidth]{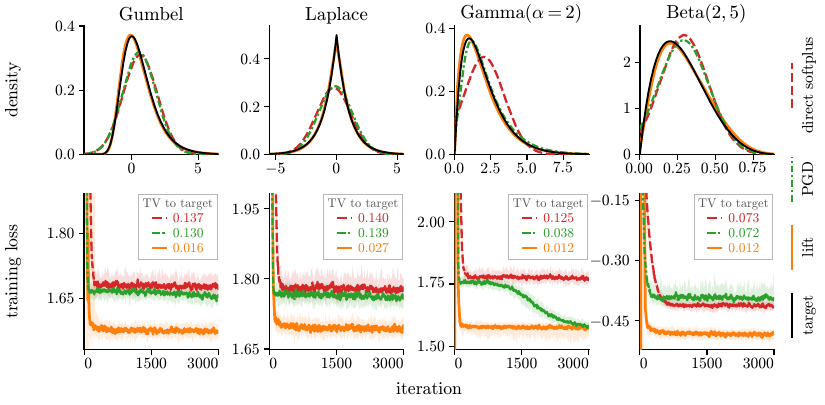}
\caption{Median-seed densities \textbf{(top)} and the training objective
\textbf{(bottom)}, the median bold inside the seed band. The framed table
lists the total-variation distance to the target at the best held-out loss.}
\label{fig:densities}
\end{figure}

\begin{figure}[t]
\centering
\includegraphics[width=\linewidth]{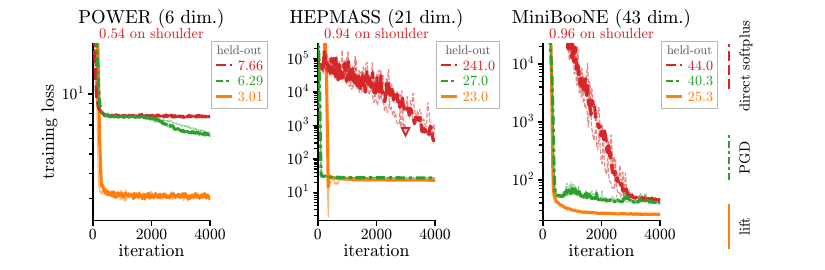}
\caption{Training objective on the tabular datasets, seeds thin and medians
bold. Red is \direct{}'s share of weights on the shoulder at the end, and the framed table
lists the best held-out loss.}
\label{fig:tabular-loss}
\end{figure}

\paragraph{Tabular datasets.}
POWER, HEPMASS~\citep{baldi2016parameterized}, and
MiniBooNE~\citep{roe2005boosted} are the tabular benchmarks of
\citet{papamakarios2017maf} in 6, 21, and 43 dimensions, run with that
paper's splits and preprocessing. On all three the \lift{} is below both baselines on every paired seed that
returns a likelihood, on the training objective and on the held-out loss
alike, by margins of nats rather than hundredths of a nat
(\cref{fig:tabular-loss}). At these widths \direct{} starts with half to two thirds of its constrained
weights on the shoulder, and training raises that share. We report each method at the configuration that gave it its best
result (\cref{tab:schedules}). While the \lift{} needs a ramp on the step in 21 and
43 dimensions, the same ramp costs \direct{} a factor of several hundred on
HEPMASS.
\directcap{} returns a non-physical likelihood on one HEPMASS seed.

\paragraph{Visits to the shoulder.} \Cref{fig:dwell-texture} measures how long a
visit to the shoulder lasts on the energy-based models, the empirical counterpart of the time below the barrier top that
\cref{thm:main}\textup{(iii)} governs for the \lift{} and
\direct{}. The three methods separate there by orders of magnitude. Under
\direct{} a visit essentially never ends. On every target but Beta the share of
weights on the shoulder is non-decreasing on all thirty seeds, and a coordinate
that enters records essentially one visit, which is what an absorbing set gives.
\pgd{} reaches the boundary as exact zeros rather than as the shoulder, and
nearly half of its visits end within a single snapshot interval. Its visits
keep ending at about the same rate across the run, which is the memoryless law
of the projected recursion in \cref{prop:clamp}. The \lift{} lies between the
two, and its visits end fastest just after entry and at a lower rate
thereafter. Some of its weights leave the shoulder, and at each of three thresholds on the derivative it still ends with more weights there than
\direct{} on every paired seed, the better fit having most of its weights near zero (\cref{fig:thresholds}). What separates the three is therefore how soon a visit ends, not how many weights end
there, and a visit of the \lift{} ends sooner at each threshold on every target but Beta. On Beta the fit pushes the weights of \direct{} off the
shoulder, so its share on the shoulder falls over training and it leaves about
as often as the \lift{} does. It is also the target \direct{} fits best. A squared map, under the same setup as the other methods, stays behind the \lift{} on every one-dimensional target.

\begin{figure}[t]
\centering
\includegraphics[width=\linewidth]{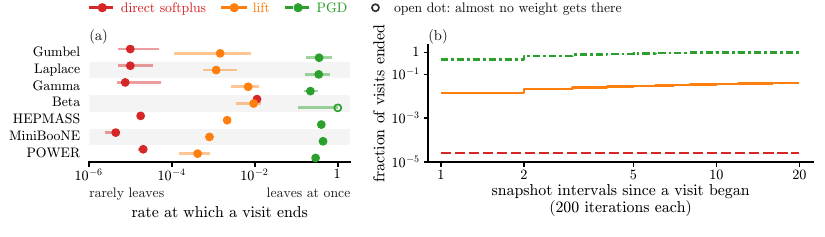}
\caption{(a) The rate at which a visit ends, from the shoulder for the softplus
methods and from the active set for \pgd{}. (b) Fraction of those visits ended by
$k$ snapshot intervals, on HEPMASS.}
\label{fig:dwell-texture}
\end{figure}

 \subsection{Convex potential flows}
\label{sec:exp-cpflow}

A convex potential flow~\citep{huang2021cpflow} transports the data through
the gradient of a strongly convex ICNN potential and evaluates its log-density
exactly, with no estimated normalizer. We compare the \lift{} with \direct{} here.

The 8-Gaussians, 2-spirals, and one-moon targets are trained at CP-Flow's
published recipe on five paired seeds, with the \lift{} replacing the
constrained weights of each block and every number read at the final iterate.
CP-Flow's published initialization places the latent weights well above the shoulder, and training moves few of them into it, so almost no weight sits there for the \lift{} to
act on and we expect no separation. None appears in the held-out loss
(\cref{fig:cpflow-toy}). On the training objective the \lift{} ends below \direct{} on every seed of
8-Gaussians and one-moon and level with it at the median on 2-spirals. The
held-out differences run a few hundredths of a nat either way, except on one
2-spirals seed, where the \lift{} is a tenth of a nat higher. 

\begin{figure}[t]
\centering
\includegraphics[width=\linewidth]{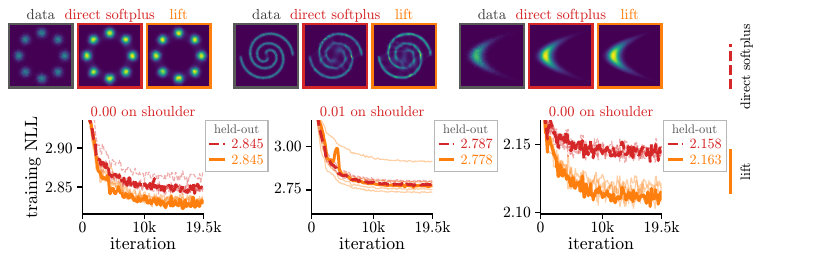}
\caption{Densities of the data and of the two methods at their median
seeds, and the training objective of five paired seeds with the held-out loss
framed.}
\label{fig:cpflow-toy}
\end{figure}

 \subsection{The Darcy transport map}
\label{sec:exp-darcy}

A partially input-convex map~\citep{wang2024conditional}, an ICNN in the
unknown, represents the
posterior of the Darcy flow inverse
problem~\citep{baldassari2024conditional}: 100
Karhunen--Lo\`eve coefficients of the log-permeability given 33 noisy
readings of the hydraulic head, trained on 50{,}000 pairs with an exact likelihood over ten seeds
(\cref{app:details}). This problem is the paper's control. No constrained weight of \direct{}
reaches the shoulder at any evaluation on any of the ten seeds, and the lowest
latent weight stays two units above the threshold defined in
\cref{sec:shoulder}, so no weight sits there for the \lift{} to act on.
Nothing separates the three. Their
held-out likelihoods tie within the seed spread and their training curves lie
on top of each other (\cref{fig:darcy-body}).

The tie does not depend on where the run starts. We repeat the comparison from the map's published
initialization, the PyTorch default, and a start matched across the three methods, at five seeds in
each of the nine cells. The nine cells agree to a fifth of a nat. \directcap{} has no weight on the shoulder in any of its fifteen runs, and the
\lift{} stays off the shoulder from every start. The start does
change \pgd{}, which from its own published initialization holds more than a
quarter of its constrained coordinates at zero on the first step and still
above one percent of them on the last. While that set of held coordinates moves by two orders of magnitude between the three starts, the held-out
likelihood does not move with it. We render its loss landscape in \cref{sec:exp-landscape}.

\begin{figure}[t]
\centering
\includegraphics[width=\linewidth]{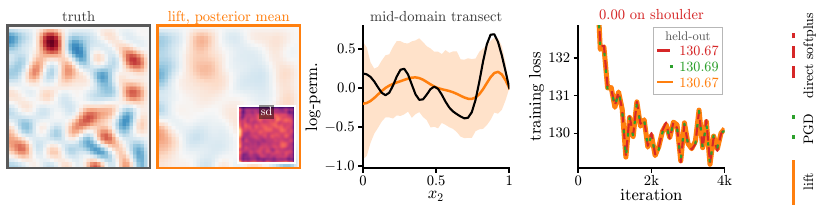}
\caption{One held-out Darcy observation: the truth, the \lift{}'s posterior
mean, the mid-domain transect, and the training objective of the three
methods with the held-out loss framed.}
\label{fig:darcy-body}
\end{figure}
 \subsection{An image regularizer}
\label{sec:exp-ffhq}

Beyond density models and transport maps, an ICNN serves as the learned
regularizer of an inverse problem, where it scores an image and is trained
adversarially~\citep{lunz2018adversarial}. We reproduce the convex regularizer of
\citet{ehrhardt2026primaldual} on grayscale FFHQ~\citep{karras2019style} inpainting at
$256\times256$ as published and change only how positivity is enforced. The regularizer is convex on the image itself, and each is scored by the PSNR
of its primal--dual reconstruction on the 20 test images (\cref{fig:ffhq}; details in \cref{app:details}). The baselines keep the published recipe's own
initialization and the \lift{} a hypernetwork initialized to match it, so all three begin at the attenuated scale that recipe gives
the baselines. \directcap{} then ends with all of its constrained weights on the shoulder,
deeper than on any other problem here. We report each method at the iterate of its highest validation PSNR.

The \lift{} reaches that iterate about three times sooner than \direct{}
does at the median, and sooner on every one of twenty-four paired seeds
(\cref{fig:ffhq-resilience}(a)), at a
test PSNR that matches \direct{} and stands above \pgd{}. Under \direct{} that PSNR arrives only when the budget ends, and as early as
the \lift{}'s iterate only at four times the published step
(\cref{fig:ffhq-steps}). The median latent weight climbs at about the step size per iteration for
\direct{} and five times faster for the \lift{}, which
\cref{fig:ffhq-resilience}(b) shows in the constrained weight it gives. The
training objective follows that weight (\cref{fig:ffhq-resilience}(c)): it falls
by the thousandth iteration under the \lift{} and only toward the end of the
budget under \direct{}. From the initialization \citet{hoedt2023principled}
prescribe instead, \direct{} does not train this
regularizer at all within the published budget, while \pgd{} trains on every
seed. The \lift{}, with its summary normalized, trains on every seed of every cell we ran.

\begin{figure}[t]
\centering
\includegraphics[width=\linewidth]{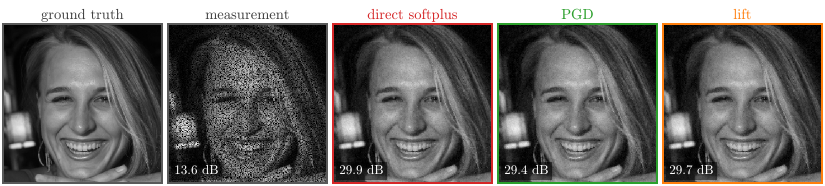}
\caption{A held-out test image, its inpainting measurement, and the
reconstruction under each regularizer at its reported iterate.}
\label{fig:ffhq}
\end{figure}

\begin{figure}[t]
\centering
\includegraphics[width=\linewidth]{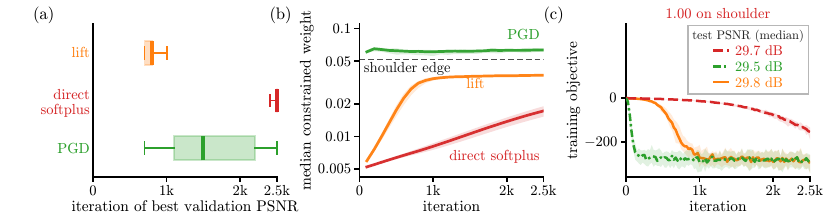}
\caption{On the image regularizer at the published step size, (a) the iteration of best validation PSNR of each method, as a box plot over seeds with whiskers at the tenth
and ninetieth percentiles, on an axis that ends at the iteration budget. (b) The median constrained weight and (c) the training objective along training, each the median over seeds inside the band of the tenth to ninetieth percentiles. Red is \direct{}'s share of weights on the shoulder at the end, and the framed table lists the median test PSNR.}
\label{fig:ffhq-resilience}
\end{figure}

\section{Related work}
\label{sec:related}

While \pgd{}~\citep{amos2017icnn} identifies the
coordinates at zero under exact
gradients~\citep{bertsekas1999nonlinear,sun2019manifold}, the projected
stochastic iterate does not~\citep{lee2012manifold,duchi2021asymptotic}. A positivity
map, i.e., softplus~\citep{huang2021cpflow,hundrieser2026hycnn} or
exponential~\citep{sill1998monotonic,hoedt2023principled}, is applied to a free latent weight,
and gradient descent on that weight corresponds, in continuous time, to mirror
descent on $\btheta$~\citep{amid2020reparameterizing}.
\citet{hoedt2023principled} find that an exponential map performs less well
than the projection and trace it to the learning dynamics, which is
where our results locate it too. \citet{sivaprasad2021curious} keep the stored weight and return a negative one
to a small positive value after each step, which, like \pgd{}, leaves no derivative of a
map in the gradient. On their classification task the \lift{} reaches a higher
median accuracy than this rule (\cref{app:cifar}). Monotone transport maps and
flows~\citep{ahn2022invertible,baptista2024monotone} obtain monotonicity from
the architecture instead, and \citet{hundrieser2026hycnn} widen each layer with
maxout branches under an initialization of their own. Each of these maps and architectures acts on the starting point, the
architecture, or the map itself, and the gradient on the shoulder remains
scaled by its derivative. The \lift{} keeps the map and
borrows the extended formulations of PDE-constrained inverse
problems~\citep{vanleeuwen2013,symes2020source,aghazade2026svgd,siahkoohi2026admm},
which add an auxiliary unknown and return to the reduced problem by a penalty or a
multiplier. Here the batch produces that unknown and there is no penalty,
through a hypernetwork~\citep{ha2017hypernetworks,mayer2024fairness,thatipelli2026hypernetwork} on a permutation-invariant
summary~\citep{zaheer2017deepsets} placed inside the constraint. Annealed and injected
noise~\citep{mobahi2015theoretical,neelakantan2015adding} is scheduled or
independent of the batch, batch normalization~\citep{ioffe2015batchnorm}
standardizes activations downstream of the weights, and the diffusion view of
stochastic gradient descent~\citep{li2017sde} describes the noise a
method already has. The fluctuation
\eqref{eq:fluct} of the \lift{} instead enters the latent weight before the
map and is coupled to the gradient through the shared batch.
 \section{Conclusion}
\label{sec:discussion}
\label{sec:conclusion}

We have presented the \lift{}, which lets a constrained weight of an ICNN on the
shoulder of its positivity map cross the barriers of the loss sooner, by
emitting the latent weight from the same batch that forms the gradient. Where this coupling aligns
positively with the loss curvature, the constant in the exponent of the time a weight on the shoulder takes to cross
a barrier of the loss falls, and removing the slack, the batch dependence, or
the shared batch sets that coupling to zero. On the energy-based models, where
\direct{} reaches the shoulder, the \lift{} fits closer to the target than \direct{} on every paired run that returns
a likelihood, and on the image regularizer it reaches the same reconstruction
about three times sooner. On the convex potential flows and the transport map, where almost no weight
reaches the shoulder, the methods we run agree. The same lift applies to any smooth monotone positivity map whose infimum is zero.
Future work involves extending the parameter space of deep network optimization
problems by the same mechanism, and asking whether the coupled fluctuation it
introduces carries the iterate toward wider minima.

\section*{AI use statement}
Generative AI tools were used to aid and polish the writing of this paper,
and to assist in proving its mathematical claims. All AI-assisted work was reviewed by the
author, who takes responsibility for the final content of this work, including
text, claims, and artifacts produced with the aid of generative AI.

\section*{Reproducibility statement}
The code that produces every experiment reported here, its configurations and
its tests, together with the machine-checked development of the formal results,
is at \url{https://github.com/luqigroup/lift}. The architecture, the
optimizer and the budget of each experiment are in \cref{app:details}, and the
proofs with the assumptions they use are in \cref{app:proofs,app:assumptions}.
 
\bibliographystyle{iclr2027_conference}
\bibliography{refs}

@inproceedings{amos2017icnn,
  author = {Amos, Brandon and Xu, Lei and Kolter, J. Zico},
  title = {Input convex neural networks},
  booktitle = {International Conference on Machine Learning},
  year = {2017},
}

@inproceedings{amos2023amortizing,
  author = {Amos, Brandon},
  title = {On amortizing convex conjugates for optimal transport},
  booktitle = {International Conference on Learning Representations},
  year = {2023},
}

@inproceedings{baldassari2024conditional,
  author = {Baldassari, Lorenzo and Siahkoohi, Ali and Garnier, Josselin and Solna, Knut and de Hoop, Maarten V.},
  title = {Conditional score-based diffusion models for {B}ayesian inference in infinite dimensions},
  booktitle = {Advances in Neural Information Processing Systems},
  year = {2023},
}

@inproceedings{bunne2022supervised,
  author = {Bunne, Charlotte and Krause, Andreas and Cuturi, Marco},
  title = {Supervised training of conditional {M}onge maps},
  booktitle = {Advances in Neural Information Processing Systems},
  year = {2022},
}

@article{cui2024unfolding,
  author = {Cui, Zhuo-Xu and Zhu, Qingyong and Cheng, Jing and Zhang, Bo and Liang, Dong},
  title = {Deep unfolding as iterative regularization for imaging inverse problems},
  journal = {Inverse Problems},
  volume = {40},
  number = {2},
  pages = {025011},
  year = {2024},
}

@article{hanggi1990reaction,
  author = {H{\"a}nggi, Peter and Talkner, Peter and Borkovec, Michal},
  title = {Reaction-rate theory: fifty years after {K}ramers},
  journal = {Reviews of Modern Physics},
  volume = {62},
  number = {2},
  pages = {251--341},
  year = {1990},
}

@inproceedings{hoedt2023principled,
  author = {Hoedt, Pieter-Jan and Klambauer, G{\"u}nter},
  title = {Principled weight initialisation for input-convex neural networks},
  booktitle = {Advances in Neural Information Processing Systems},
  year = {2023},
}

@inproceedings{huang2021cpflow,
  author = {Huang, Chin-Wei and Chen, Ricky T. Q. and Tsirigotis, Christos and Courville, Aaron},
  title = {Convex potential flows: universal probability distributions with optimal transport and convex optimization},
  booktitle = {International Conference on Learning Representations},
  year = {2021},
}

@inproceedings{kan2022mqf,
  author = {Kan, Kelvin and Aubet, Fran{\c c}ois-Xavier and Januschowski, Tim and Park, Youngsuk and Benidis, Konstantinos and Ruthotto, Lars and Gasthaus, Jan},
  title = {Multivariate quantile function forecaster},
  booktitle = {International Conference on Artificial Intelligence and Statistics},
  year = {2022},
}

@inproceedings{korotin2021benchmark,
  author = {Korotin, Alexander and Li, Lingxiao and Genevay, Aude and Solomon, Justin and Filippov, Alexander and Burnaev, Evgeny},
  title = {Do neural optimal transport solvers work? {A} continuous {W}asserstein-2 benchmark},
  booktitle = {Advances in Neural Information Processing Systems},
  year = {2021},
}

@article{kramers1940brownian,
  author = {Kramers, Hendrik A.},
  title = {{B}rownian motion in a field of force and the diffusion model of chemical reactions},
  journal = {Physica},
  volume = {7},
  number = {4},
  pages = {284--304},
  year = {1940},
}

@inproceedings{lin2023falcon,
  author = {Lin, Alexander and Ba, Demba},
  title = {How to train your {FALCON}: learning log-concave densities with energy-based neural networks},
  booktitle = {Symposium on Advances in Approximate Bayesian Inference},
  year = {2023},
}

@inproceedings{taghvaei2022bayes,
  author = {Taghvaei, Amirhossein and Hosseini, Bamdad},
  title = {An optimal transport formulation of {B}ayes' law for nonlinear filtering algorithms},
  booktitle = {IEEE Conference on Decision and Control},
  year = {2022},
}

@inproceedings{vesseron2025moment,
  author = {Vesseron, Nina and B{\'e}thune, Louis and Cuturi, Marco},
  title = {Sample and map from a single convex potential: generation using conjugate moment measures},
  booktitle = {Advances in Neural Information Processing Systems},
  year = {2025},
}

@article{wang2024conditional,
  author = {Wang, Zheyu Oliver and Baptista, Ricardo and Marzouk, Youssef and Ruthotto, Lars and Verma, Deepanshu},
  title = {Efficient neural network approaches for conditional optimal transport with applications in {B}ayesian inference},
  journal = {SIAM Journal on Scientific Computing},
  volume = {47},
  number = {4},
  pages = {C979--C1005},
  year = {2025},
}

@inproceedings{makkuva2020icnnot,
  author = {Makkuva, Ashok Vardhan and Taghvaei, Amirhossein and Oh, Sewoong and Lee, Jason},
  title = {Optimal transport mapping via input convex neural networks},
  booktitle = {International Conference on Machine Learning},
  year = {2020},
}

@inproceedings{papamakarios2017maf,
  author = {Papamakarios, George and Pavlakou, Theo and Murray, Iain},
  title = {Masked autoregressive flow for density estimation},
  booktitle = {Advances in Neural Information Processing Systems},
  year = {2017},
}

@misc{siahkoohi2026admm,
  author = {Siahkoohi, Ali and Aghazade, Kamal and Gholami, Ali},
  title = {Dual-space posterior sampling for {B}ayesian inference in constrained inverse problems},
  howpublished = {arXiv preprint arXiv:2603.00393},
  year = {2026},
}

@article{vanleeuwen2013,
  author = {van Leeuwen, Tristan and Herrmann, Felix J.},
  title = {Mitigating local minima in full-waveform inversion by expanding the search space},
  journal = {Geophysical Journal International},
  volume = {195},
  number = {1},
  pages = {661--667},
  year = {2013},
}

@inproceedings{zaheer2017deepsets,
  author = {Zaheer, Manzil and Kottur, Satwik and Ravanbakhsh, Siamak and P{\'o}czos, Barnab{\'a}s and Salakhutdinov, Ruslan and Smola, Alexander J.},
  title = {Deep sets},
  booktitle = {Advances in Neural Information Processing Systems},
  year = {2017},
}

@inproceedings{li2018visualizing,
  author = {Li, Hao and Xu, Zheng and Taylor, Gavin and Studer, Christoph and Goldstein, Tom},
  title = {Visualizing the loss landscape of neural nets},
  booktitle = {Advances in Neural Information Processing Systems},
  year = {2018},
}

@inproceedings{li2017sde,
  author = {Li, Qianxiao and Tai, Cheng and E, Weinan},
  title = {Stochastic modified equations and adaptive stochastic gradient algorithms},
  booktitle = {International Conference on Machine Learning},
  year = {2017},
}

@techreport{mayer2024fairness,
  author = {Mayer, Paul M. and Luzi, Lorenzo and Siahkoohi, Ali and Johnson, Don H. and Baraniuk, Richard G.},
  title = {Improving fairness and mitigating {MAD}ness in generative models},
  institution = {Rice University},
  year = {2024},
}

@article{baldi2016parameterized,
  author = {Baldi, Pierre and Cranmer, Kyle and Faucett, Taylor and Sadowski, Peter and Whiteson, Daniel},
  title = {Parameterized neural networks for high-energy physics},
  journal = {European Physical Journal C},
  volume = {76},
  number = {5},
  pages = {235},
  year = {2016},
}

@book{bertsekas1999nonlinear,
  author = {Bertsekas, Dimitri P.},
  title = {Nonlinear Programming},
  publisher = {Athena Scientific},
  year = {1999},
}

@inproceedings{symes2020source,
  author = {Symes, William W. and Chen, Huiyi and Minkoff, Susan E.},
  title = {Full-waveform inversion by source extension: why it works},
  booktitle = {SEG Technical Program Expanded Abstracts},
  pages = {765--769},
  year = {2020},
}

@article{baptista2024monotone,
  author = {Baptista, Ricardo and Marzouk, Youssef and Zahm, Olivier},
  title = {On the representation and learning of monotone triangular transport maps},
  journal = {Foundations of Computational Mathematics},
  volume = {24},
  pages = {2063--2108},
  year = {2024},
}

@misc{thatipelli2026hypernetwork,
  author = {Thatipelli, Anirudh and Siahkoohi, Ali},
  title = {Hypernetwork-based approach for grid-independent functional data clustering},
  howpublished = {arXiv preprint arXiv:2602.22823},
  year = {2026},
}

@inproceedings{ha2017hypernetworks,
  author = {Ha, David and Dai, Andrew M. and Le, Quoc V.},
  title = {{H}yper{N}etworks},
  booktitle = {International Conference on Learning Representations},
  year = {2017},
}

@inproceedings{kingma2015adam,
  author = {Kingma, Diederik P. and Ba, Jimmy},
  title = {{A}dam: a method for stochastic optimization},
  booktitle = {International Conference on Learning Representations},
  year = {2015},
}

@article{hundrieser2026hycnn,
  title={Hyper Input Convex Neural Networks for Shape Constrained Learning
         and Optimal Transport},
  author={Hundrieser, Shayan and Kong, Insung and Schmidt-Hieber, Johannes},
  journal={arXiv preprint arXiv:2604.26942},
  year={2026}
}

@techreport{krizhevsky2009cifar,
  title={Learning Multiple Layers of Features from Tiny Images},
  author={Krizhevsky, Alex},
  institution={University of Toronto},
  year={2009}
}

@inproceedings{sivaprasad2021curious,
  title={The Curious Case of Convex Neural Networks},
  author={Sivaprasad, Sarath and Singh, Ankur and Manwani, Naresh and Gandhi, Vineet},
  booktitle={Joint European Conference on Machine Learning and Knowledge Discovery in Databases},
  year={2021}
}

@inproceedings{mobahi2015theoretical,
  title={A Theoretical Analysis of Optimization by {G}aussian Continuation},
  author={Mobahi, Hossein and Fisher III, John W.},
  booktitle={AAAI Conference on Artificial Intelligence},
  year={2015}
}

@inproceedings{neelakantan2015adding,
  author = {Neelakantan, Arvind and Vilnis, Luke and Le, Quoc V. and Sutskever, Ilya and Kaiser, Lukasz and Kurach, Karol and Martens, James},
  title = {Adding gradient noise improves learning for very deep networks},
  booktitle = {International Conference on Learning Representations, Workshop Track},
  year = {2016},
}

@article{ehrhardt2026primaldual,
  title = {A primal-dual algorithm for image reconstruction with input-convex
           neural network regularizers},
  author = {Ehrhardt, Matthias J. and Mukherjee, Subhadip and Wong, Hok Shing},
  journal = {SIAM Journal on Imaging Sciences},
  volume = {19},
  number = {2},
  pages = {782--806},
  year = {2026},
}

@inproceedings{kruse2021hint,
  author = {Kruse, Jakob and Detommaso, Gianluca and K{\"o}the, Ullrich and Scheichl, Robert},
  title = {{HINT}: hierarchical invertible neural transport for density estimation and {B}ayesian inference},
  booktitle = {Proceedings of the AAAI Conference on Artificial Intelligence},
  volume = {35},
  pages = {8191--8199},
  year = {2021},
}

@article{lee2012manifold,
  author  = {Lee, Sangkyun and Wright, Stephen J.},
  title   = {Manifold identification in dual averaging for regularized stochastic online learning},
  journal = {Journal of Machine Learning Research},
  volume  = {13},
  number  = {55},
  pages   = {1705--1744},
  year    = {2012},
}

@article{duchi2021asymptotic,
  author  = {Duchi, John C. and Ruan, Feng},
  title   = {Asymptotic optimality in stochastic optimization},
  journal = {The Annals of Statistics},
  volume  = {49},
  number  = {1},
  pages   = {21--48},
  year    = {2021},
}

@article{roe2005boosted,
  title   = {Boosted decision trees as an alternative to artificial neural networks for particle identification},
  author  = {Roe, Byron P. and Yang, Hai-Jun and Zhu, Ji and Liu, Yong and Stancu, Ion and McGregor, Gordon},
  journal = {Nuclear Instruments and Methods in Physics Research Section A},
  volume  = {543}, number = {2--3}, pages = {577--584}, year = {2005}
}

@inproceedings{sun2019manifold,
  author    = {Sun, Yifan and Jeong, Halyun and Nutini, Julie and Schmidt, Mark},
  title     = {Are we there yet? {M}anifold identification of gradient-related proximal methods},
  booktitle = {International Conference on Artificial Intelligence and Statistics},
  year      = {2019},
}

@inproceedings{ioffe2015batchnorm,
  author    = {Ioffe, Sergey and Szegedy, Christian},
  title     = {Batch normalization: accelerating deep network training by reducing internal covariate shift},
  booktitle = {International Conference on Machine Learning},
  pages     = {448--456},
  year      = {2015},
}

@inproceedings{amid2020reparameterizing,
  author    = {Amid, Ehsan and Warmuth, Manfred K.},
  title     = {Reparameterizing mirror descent as gradient descent},
  booktitle = {Advances in Neural Information Processing Systems},
  year      = {2020},
}

@inproceedings{ahn2022invertible,
  author    = {Ahn, Byeongkeun and Kim, Chiyoon and Hong, Youngjoon and Kim, Hyunwoo J.},
  title     = {Invertible monotone operators for normalizing flows},
  booktitle = {Advances in Neural Information Processing Systems},
  year      = {2022},
}

@inproceedings{sill1998monotonic,
  author    = {Sill, Joseph},
  title     = {Monotonic networks},
  booktitle = {Advances in Neural Information Processing Systems},
  year      = {1997},
}

@inproceedings{lunz2018adversarial,
  author    = {Lunz, Sebastian and {\"O}ktem, Ozan and Sch{\"o}nlieb, Carola-Bibiane},
  title     = {Adversarial regularizers in inverse problems},
  booktitle = {Advances in Neural Information Processing Systems},
  year      = {2018},
}

@inproceedings{mukherjee2020learned,
  author  = {Mukherjee, Subhadip and Dittmer, S{\"o}ren and Shumaylov, Zakhar and Lunz, Sebastian and {\"O}ktem, Ozan and Sch{\"o}nlieb, Carola-Bibiane},
  title   = {Data-driven convex regularizers for inverse problems},
  booktitle = {{IEEE} International Conference on Acoustics, Speech and Signal Processing},
  year    = {2024},
}

@article{spitzer1956,
  author  = {Spitzer, Frank},
  title   = {A combinatorial lemma and its application to probability theory},
  journal = {Transactions of the American Mathematical Society},
  volume  = {82},
  number  = {2},
  pages   = {323--339},
  year    = {1956},
}

@book{asmussen2003applied,
  author    = {Asmussen, S{\o}ren},
  title     = {Applied Probability and Queues},
  publisher = {Springer},
  year      = {2003},
}

@inproceedings{aghazade2026svgd,
  author = {Aghazade, Kamal and Siahkoohi, Ali and Gholami, Ali},
  title = {Scalable {B}ayesian full waveform inversion via dual augmented {L}agrangian {SVGD}},
  booktitle = {SEG Technical Program Expanded Abstracts},
  year = {2026},
}

@inproceedings{karras2019style,
  author = {Karras, Tero and Laine, Samuli and Aila, Timo},
  title = {A style-based generator architecture for generative adversarial networks},
  booktitle = {{IEEE} Conference on Computer Vision and Pattern Recognition},
  year = {2019},
}

@book{freidlin2012random,
  author = {Freidlin, Mark I. and Wentzell, Alexander D.},
  title = {Random Perturbations of Dynamical Systems},
  publisher = {Springer},
  year = {2012},
}

@article{berthier2023incremental,
  author = {Berthier, Rapha\"{e}l},
  title = {Incremental learning in diagonal linear networks},
  journal = {Journal of Machine Learning Research},
  volume = {24},
  number = {171},
  pages = {1--26},
  year = {2023},
}

@article{levin2025effect,
  author = {Levin, Eitan and Kileel, Joe and Boumal, Nicolas},
  title = {The effect of smooth parametrizations on nonconvex optimization landscapes},
  journal = {Mathematical Programming},
  volume = {209},
  pages = {63--111},
  year = {2025},
}

\appendix

\section{Proofs}
\label{app:proofs}

This appendix records the attribution and the discharge of the hypotheses of
\cref{sec:diffusion}, proves \cref{thm:joint-necessity}, states and proves
\cref{lem:diffusion,cor:fpt}, and closes the escape
argument with the proof of \cref{thm:main}, which collects
\cref{lem:diffusion}\textup{(ii)}--\textup{(iii)},\textup{(vi)} and
\cref{cor:fpt,cor:dwell}. Three inputs are modeling steps, and we name each at
the point where it enters, with the mathematics that follows once it is in
place. The first is the
It\^o process for the discrete recursion, the register of \textup{(A1)},
which follows the stochastic modified equations of \citet{li2017sde}; the It\^o process is the primitive of every diffusion statement below, and nothing is
proved about the passage from the recursion to it. That approximation is weak
and holds over horizons of order $1/\eta$; the times of \cref{cor:fpt} are
exponentially longer. The second is Dynkin's formula, which turns the mean exit time of the diffusion along $\be_b$ into the boundary-value problem of
\cref{app:proof-fpt}; that problem has exactly one continuous solution, and
everything downstream of it is proved. The third is the same formula at the
indicator of $[a,w_s]$, which turns the mean occupation time into the boundary-value
problem of \cref{cor:dwell}, one modeling input applied to a second functional.
A fourth step is an approximation, i.e., we read the effective diffusion as constant across the barrier, and \cref{app:proof-diffusion} brackets the
error of this reading at the level of the effective action.

The analysis describes training as it runs by
\eqref{eq:lift-objective}, with one batch and one latent weight per step,
and adds nothing to it. Its expectations are over the batch law and are
never formed, and the one estimate it uses is a diagnostic computed outside
training.
The notation is that of \cref{sec:diffusion}: at an iterate with $\bb$ and
$\bphi$ held, $\delta\tilde\btheta$ is the fluctuation of \eqref{eq:fluct},
$\bar{\tilde\btheta}=\bb+\E_\bX[h_\bphi(\bX)]$ the mean latent weight,
$\bg_0(\bX)=\nabla_{\tilde\btheta}\L(\bar{\tilde\btheta};\bX)$ the gradient of a
batch at that iterate, $\bxi=\bg_0-\E_\bX[\bg_0]$ the minibatch noise there,
$\bSigma_\xi=\E_\bX[\delta\tilde\btheta\,\bxi^\top]$ the coupling,
$\bH_{\tilde\btheta}=\E_\bX[\nabla^2_{\tilde\btheta}\L(\bar{\tilde\btheta};\bX)]$
the mean loss curvature, $\bV=\Cov_\bX[\delta\tilde\btheta]$ the
iterate-fluctuation covariance, and $\br_2$ the remainder of the chain of \cref{sec:mechanism}. Since $\Cov_\bX[\bH_{\tilde\btheta}\delta\tilde\btheta,\br_2]=\bH_{\tilde\btheta}\Cov_\bX[\delta\tilde\btheta,\br_2]$,
the remainder $\varrho$ of \cref{sec:mechanism} is $\varrho(\be_b)$, with
\[
\varrho(\be)=2\,\be^\top\Cov_\bX[\bg_0,\br_2]\,\be+2\,\be^\top\bH_{\tilde\btheta}\Cov_\bX[\delta\tilde\btheta,\br_2]\,\be+\be^\top\Cov_\bX[\br_2,\br_2]\,\be
\]
for a unit direction $\be$, the form the proofs use. Two conventions fix the rest. The positivity-map slope of
\textup{(A2)} enters
the per-step update as a scalar prefactor multiplying every term, and the
identity \eqref{eq:sigma-eff} is that decomposition with the prefactor set to
one, because the latent-weight gradient $\bg$ of \eqref{eq:chain} already
includes $\psi'$ and the \textcolor{liftdirect}{direct-softplus} value $\sigma_s^2\sigma^2_\mathrm{obj}$
is the variance of that gradient. Where the prefactor is kept explicit, as it
is in \cref{app:proof-pgd}, the excess terms are the paper's divided by
$\sigma_s^2$, and we say so at the one place where the conversion matters. The
estimator of \eqref{eq:cross-cov} is a $T$-probe statistic with Bessel's
divisor; on the same data its trailing-window form with divisor $1/T$ differs
from it by the factor $T/(T-1)$, so every statement about one transports to
the other, and the proof of \cref{thm:joint-necessity} is written on the probe
form.

\subsection{The hypotheses}
\label{app:assumptions}

Each result names what it consumes, from
Assumptions~\ref{as:A1}--\ref{as:A2} of \cref{sec:hypotheses}.
\Cref{tab:hyp} records the attribution,
and the two clauses of \textup{(A2)} are the only hypotheses with content
beyond regularity.

\begin{table}[!ht]
\centering
\caption{Every hypothesis of this appendix and the results that consume it.}
\label{tab:hyp}
\small
\setlength{\tabcolsep}{4pt}
\begin{tabular}{l>{\raggedright\arraybackslash}p{0.46\linewidth}>{\raggedright\arraybackslash}p{0.31\linewidth}}
\toprule
 & statement & consumed by \\
\midrule
\textup{(A1)} & It\^o process; i.i.d.\ batches; $\bSigma$ Lipschitz & \cref{lem:diffusion}\textup{(iv)}--\textup{(v)}, \cref{cor:fpt,cor:dwell} \\
 & the i.i.d.\ clause alone & \cref{thm:joint-necessity}, sample form \\
regularity & $\bg_0,\delta\tilde\btheta,\br_2,\nabla^2_{\tilde\btheta}\L(\bar{\tilde\btheta};\bX)$ square integrable; $\bH_{\tilde\btheta}^{\!\top}=\bH_{\tilde\btheta}$ & \cref{lem:diffusion}\textup{(i)}--\textup{(iii)},\textup{(vi)} \\
 & $\sigma^2_\mathrm{obj}>0$ & the ordering of \cref{cor:fpt,cor:dwell,prop:precond-full} \\
 & $\tilde\L'$, $\sigma^2_\mathrm{eff}$ continuous on $\R$; $\sigma^2_\mathrm{eff}>0$ & \cref{lem:diffusion}\textup{(iv)}--\textup{(v)} \\
 & $\tilde\btheta$ and $\bg_0$ of finite fourth moment & \cref{thm:joint-necessity} \\
\textup{(A2a)} & $|\psi'(w)-\sigma_s|\le\varepsilon$ on $\mathcal{B}$, $0\le\varepsilon<\sigma_s$ & \cref{lem:diffusion} \\
\textup{(A2b)} & $\alpha=\tilde\L(w_s)-\tilde\L(w_b)>0$, extremal on $[a,b]$ & \cref{cor:fpt,cor:dwell,thm:directional,prop:precond-full} \\
$c>0$ & at the one iterate and the one direction $\be_b$ & the ordering of \cref{cor:fpt,cor:dwell,prop:precond-full} \\
Dynkin & the mean exit time solves the boundary-value problem at every step size, continuous up to the closed interval & \cref{cor:fpt,thm:directional,prop:precond-full} \\
 & the occupation time likewise, continuously differentiable on $(a,b)$ & \cref{cor:dwell,prop:precond-full} \\
preconditioned & \textup{(A1)} with $\bP_k$ fixed by earlier batches; $p$ constant across the barrier & \cref{prop:precond-full} \\
projection & $\bSigma_\mathrm{PGD}=\bSigma_\mathrm{obj}$ & \cref{thm:directional,thm:pgd-zero} \\
 & a common mean latent weight & \cref{thm:directional,thm:pgd-zero} \\
\bottomrule
\end{tabular}
\end{table}

\textup{(A2a)} brackets the \textcolor{liftdirect}{direct-softplus} diffusion,
\begin{equation*}
(\sigma_s-\varepsilon)^2\sigma^2_\mathrm{obj}\;\le\;\psi'(w)^2\sigma^2_\mathrm{obj}\;\le\;(\sigma_s+\varepsilon)^2\sigma^2_\mathrm{obj},
\qquad w\in\mathcal{B},
\end{equation*}
and its tolerance is the substance of the clause: at $\varepsilon=0$ the interval collapses to a point for any $\psi$ with injective derivative. For softplus
the clause is a theorem. From $\psi''\le\psi'$,
\begin{equation*}
e^{-\delta}\psi'(w)\;\le\;\psi'(u)\;\le\;e^{\delta}\psi'(w),
\qquad |u-w|\le\delta,
\end{equation*}
so for $0<\rho<1$ every interval of half-width $\log(1+\rho)$ about $w$ satisfies
\textup{(A2a)} with $\sigma_s=\psi'(w)$ and $\varepsilon=\rho\sigma_s$, at any
depth of the shoulder. Its content is therefore that the weights sit in such an interval, which is a statement about training and not about $\psi$.

The slack Jacobian $\bJ_\bb=\bm I$ is an identity of \eqref{eq:lift}, and it makes the estimate along $\be_b$ of \cref{thm:joint-necessity} the estimator
itself. The chain
$\bg=\bg_0+\bH_{\tilde\btheta}\delta\tilde\btheta+\br_2$ is an identity once
$\br_2$ is the discrepancy, with $\bxi$ of order one in the fluctuation.

\subsection{Joint necessity}
\label{app:proof-joint-necessity}

\Cref{thm:joint-necessity} is stated in \cref{sec:diffusion}, and the two
substitutions printed after it there are proved first below.

\begin{proof}
Write $\lambda$ for the divisor, $\lambda=(T-1)^{-1}$ or $\lambda=T^{-1}$, so
that $\widehat\bSigma^{(t)}=\lambda\sum_{k=1}^{T}\delta\tilde\btheta_k\,\delta\bg_k^\top$.

\emph{No learnable slack.} With $\bJ_\bb=\bm 0$, summand by summand,
\begin{equation*}
\bJ_\bb^\top\,\widehat\bSigma^{(t)}
\;=\;\lambda\sum_{k=1}^{T}\bJ_\bb^\top\,\delta\tilde\btheta_k\,\delta\bg_k^\top
\;=\;\bm 0,
\end{equation*}
at every iterate, for every $T\ge2$ at divisor $T-1$ and every $T\ge1$ at divisor $T$.

\emph{A hypernetwork constant in the batch.} A hypernetwork constant in the batch gives
$\tilde\btheta_k=\bb+h_\bphi=\bar{\tilde\btheta}$ for every $k$, so for the same $T$
\begin{equation*}
\delta\tilde\btheta_k=\tilde\btheta_k-\bar{\tilde\btheta}=\bm 0,
\qquad
\widehat\bSigma^{(t)}=\lambda\sum_{k=1}^{T}\bm 0\,\delta\bg_k^\top=\bm 0 .
\end{equation*}
The contrapositive of the two substitutions is the sentence printed after the theorem,
\begin{equation*}
\bJ_\bb^\top\widehat\bSigma^{(t)}\ne\bm 0
\;\Longrightarrow\;
\bJ_\bb\ne\bm 0\ \text{ and }\ \delta\tilde\btheta_k\ne\bm 0\ \text{ for some }k,
\end{equation*}
and neither substitution uses an assumption or a property of $\psi$.

\emph{The population zero.} Let $\mathcal G=\sigma(\bphi,\bb)$.
Conditional independence of $\tilde\btheta$ and $\bg_0$ given $\mathcal G$
passes to $\delta\tilde\btheta=\tilde\btheta-\E[\tilde\btheta\mid\mathcal G]$
and $\bxi=\bg_0-\E[\bg_0\mid\mathcal G]$, each a $\mathcal G$-measurable shift
of one factor. For measurable, integrable, conditionally independent vectors
the conditional expectation of the outer product factorizes through the
disintegration kernel, so
\begin{equation}\label{eq:pf-popzero}
\bSigma_\xi=\E\bigl[\delta\tilde\btheta\,\bxi^\top\,\big|\,\mathcal G\bigr]
=\E\bigl[\delta\tilde\btheta\,\big|\,\mathcal G\bigr]\,
\E\bigl[\bxi\,\big|\,\mathcal G\bigr]^\top=\bm 0
\qquad\text{a.s.},
\end{equation}
both factors vanishing by the centering, and $\tr\bSigma_\xi=0$ with it. At
fixed parameters \eqref{eq:pf-popzero} is the unconditional statement for
independent integrable centered vectors; independence is what produces the
zero, dependent centered pairs with $\bSigma_\xi\ne\bm 0$ being available.

\emph{The sample version.} Fix the parameters and take every
expectation under the conditional law. Sample-mean centering annihilates the
constant shifts $\E[\tilde\btheta\mid\mathcal G]$ and $\E[\bg_0\mid\mathcal G]$,
so for $T\ge1$ the estimate formed from the raw pairs
$(\tilde\btheta_k,\bg_0(\bX_k))$ is the one formed from
$(\delta\tilde\btheta_k,\bxi_k)$, and the sample-cross-covariance identity,
exact finite algebra, gives
\begin{equation}\label{eq:pf-sample-identity}
\widehat\bSigma_\xi^{(t)}
=\lambda\sum_{k=1}^{T}(\delta\tilde\btheta_k-\overline{\delta\tilde\btheta})
(\bxi_k-\bar\bxi)^\top
=\lambda\sum_{k=1}^{T}\delta\tilde\btheta_k\,\bxi_k^\top
-\lambda T\,\overline{\delta\tilde\btheta}\,\bar\bxi^{\top},
\end{equation}
with $\lambda T=T/(T-1)$ at Bessel's divisor and $\lambda T=1$ at divisor $T$. Under
\textup{(A1)} the probes are independent and identically distributed, so
\eqref{eq:pf-popzero} and the pairwise independence give
\begin{equation}\label{eq:pf-moments}
\E[\delta\tilde\btheta_k\,\bxi_k^\top]=\bm 0,
\qquad
\E[\delta\tilde\btheta_k\,\bxi_l^\top]=\bm 0\quad(k\ne l),
\end{equation}
and the expectation of \eqref{eq:pf-sample-identity}, term by term, is
\begin{equation*}
\E\bigl[\widehat\bSigma_\xi^{(t)}\bigr]
=\lambda\sum_{k=1}^{T}\E[\delta\tilde\btheta_k\,\bxi_k^\top]
-\frac{\lambda}{T}\sum_{k,l=1}^{T}\E[\delta\tilde\btheta_k\,\bxi_l^\top]=\bm 0
\end{equation*}
for every $T\ge1$ at divisor $T$ and every $T\ge2$ at Bessel's. On a nonzero target $\bSigma$ the same
two computations give $\bSigma$ at Bessel's divisor and $(T-1)\bSigma/T$ at
divisor $T$, for $T\ge2$, which is where the divisors part company.

The summands $\delta\tilde\btheta_k\bxi_k^\top$ are pairwise independent,
identically distributed, integrable, and centered by \eqref{eq:pf-moments},
as are $\delta\tilde\btheta_k$ and $\bxi_k$, so Etemadi's strong law gives
\begin{equation*}
\frac1T\sum_{k=1}^{T}\delta\tilde\btheta_k\,\bxi_k^\top\to\bm 0,
\qquad
\overline{\delta\tilde\btheta}\to\bm 0,
\qquad
\bar\bxi\to\bm 0
\qquad\text{a.s.\ as }T\to\infty,
\end{equation*}
entrywise and hence in the matrix topology, and
$\widehat\bSigma_\xi^{(t)}\to\bm 0$ a.s.\ by \eqref{eq:pf-sample-identity},
the factor $\lambda T$ tending to one.

For the rate, the variance of an average of $T$ pairwise independent
identically distributed centered summands is the summand variance over $T$.
With $\sigma_{ij}^2=\Var[(\delta\tilde\btheta\,\bxi^\top)_{ij}]$,
$\sigma_u^2=\Var[(\delta\tilde\btheta)_i]$, and $\sigma_v^2=\Var[(\bxi)_j]$,
\begin{equation}\label{eq:pf-sum-l2}
\E\Bigl[\Bigl(\lambda\sum_{k=1}^{T}(\delta\tilde\btheta_k\,\bxi_k^\top)_{ij}\Bigr)^2\Bigr]
=\lambda^2T\,\sigma_{ij}^2,
\end{equation}
a root mean square of $\sqrt T\,\sigma_{ij}/(T-1)$ at Bessel's divisor and
$\sigma_{ij}/\sqrt T$ at divisor $T$. The joint law of the $2T$ vectors is
the product over probes of the per-probe laws, each of which factorizes by
the conditional independence, so the collections $\{\delta\tilde\btheta_k\}_k$
and $\{\bxi_k\}_k$ are independent, hence so are the two sample means, and
\begin{equation}\label{eq:pf-meanprod-l2}
\E\bigl[\bigl((\overline{\delta\tilde\btheta})_i\,(\bar\bxi)_j\bigr)^2\bigr]
=\E\bigl[(\overline{\delta\tilde\btheta})_i^{\,2}\bigr]\,
\E\bigl[(\bar\bxi)_j^{\,2}\bigr]
=\frac{\sigma_u^2\sigma_v^2}{T^2}.
\end{equation}
This is the one step of the proof that uses the i.i.d.\ clause of
\textup{(A1)} beyond its pairwise form. By $(x-y)^2\le2x^2+2y^2$ applied to
\eqref{eq:pf-sample-identity} at divisor $T$, with \eqref{eq:pf-sum-l2} and
\eqref{eq:pf-meanprod-l2},
\begin{equation*}
\E\bigl[(\widehat\bSigma_\xi^{(t)})_{ij}^{2}\bigr]
\le\frac{2\sigma_{ij}^2}{T}+\frac{2\sigma_u^2\sigma_v^2}{T^2},
\end{equation*}
and Chebyshev's inequality makes the entries $O_P(T^{-1/2})$; the
variances $\sigma_{ij}^2$, $\sigma_u^2$, $\sigma_v^2$ are finite by the
fourth-moment clause of the theorem; an almost-sure rate is neither
claimed nor proved. The sample version does not
vanish at finite $T$: for independent standard Gaussian factors in one
dimension the population-centered sum at divisor $T$ has
$\E[(T^{-1}\sum_k\delta\tilde\theta_k\xi_k)^2]=1/T$, so no nonempty probe
set makes it almost surely zero. Every display above is made summand by
summand or carries the factor $\lambda T$, so it holds on the trailing window of
\eqref{eq:cross-cov} as well.
\end{proof}

\subsection{Diffusion and barrier}
\label{app:proof-diffusion}

\begin{lemma}[Diffusion and barrier; parts \textup{(ii)}--\textup{(iii)} are \cref{thm:main}\textup{(i)} and part \textup{(vi)} is \cref{thm:main}\textup{(ii)}]\label{lem:diffusion}
Let Assumptions~\ref{as:A1}--\ref{as:A2} hold, Assumption~\ref{as:A2}\textup{(b)} not being used and the It\^o process of
Assumption~\ref{as:A1} entering parts \textup{(iv)}--\textup{(v)} alone. Fix the mean latent weight
$\bar{\tilde\btheta}$ on the interval $\mathcal{B}$ and a coordinate direction $\be_b$ of
the slack, with $\be_b^\top\bar{\tilde\btheta}$ the point $\tilde w_s$ of \textup{(A2a)}, so that $\sigma_s=\psi'(\be_b^\top\bar{\tilde\btheta})$, and write $\bSigma_\mathrm{slack}=\E_\bX[\delta\tilde\btheta\,\delta\bg^\top]$
for the population cross-covariance of \eqref{eq:cross-cov}, whose trace is
$\sigma^2_\mathrm{Jac}$, $\sigma^2_\mathrm{obj}=\be_b^\top\Var_\bX[\nabla_\btheta\L]\,\be_b$
for the objective variance in the constrained coordinate along $\be_b$, and
$\sigma^2_\mathrm{eff}=\be_b^\top\Cov_\bX[\bg]\,\be_b$ for the projected
diffusion of the lifted update. Then the following hold.
\textup{(i)} $\bSigma_\mathrm{slack}=\bSigma_\xi+\bV\bH_{\tilde\btheta}+\Cov_\bX[\delta\tilde\btheta,\br_2]$.
\textup{(ii)} With no term discarded,
\begin{equation*}\tag{\ref{eq:sigma-eff}}
\sigma^2_\mathrm{eff} \;=\;
\underbrace{\sigma_s^2\sigma^2_\mathrm{obj}}_{\text{\direct{} value}}
\;+\; \underbrace{2\,\be_b^{\top}\bH_{\tilde\btheta}\bSigma_\xi\,\be_b}_{\text{coupling term}}
\;+\; \underbrace{\be_b^{\top}\bH_{\tilde\btheta}\bV\bH_{\tilde\btheta}\,\be_b}_{\text{fluctuation term}}
\;+\; \underbrace{\varrho}_{\text{remainder}},
\end{equation*}
with $\varrho=\varrho(\be_b)$ the remainder of \cref{sec:mechanism}.
\textup{(iii)} The \direct{} value at the same iterate is the first term of \eqref{eq:sigma-eff} alone, within the bracket of Assumption~\ref{as:A2}. The fluctuation term is $\Var_\bX[\langle\bH_{\tilde\btheta}\be_b,\delta\tilde\btheta\rangle]\ge0$,
so $\sigma^2_\mathrm{eff}-\sigma_s^2\sigma^2_\mathrm{obj}\ge c$ with no
hypothesis on the coupling.
\textup{(iv)} Under the It\^o process of Assumption~\ref{as:A1}, with $\tilde\L'$ and
$\sigma^2_\mathrm{eff}$ continuous on $\R$ and $\sigma^2_\mathrm{eff}>0$, the diffusion $dw=-\tilde\L'(w)\,dt+\sqrt{\eta}\,\sigma_\mathrm{eff}(w)\,dB_t$ along $\be_b$ has
the stationary density $\pi\propto\sigma^{-2}_\mathrm{eff}\exp(-(2/\eta)\,U_\mathrm{eff})$,
$U_\mathrm{eff}(w)=\int^{w}\tilde\L'(u)/\sigma^2_\mathrm{eff}(u)\,du$, and every
zero-flux stationary density is of this form.
\textup{(v)} A barrier of action $\alpha=\tilde\L(w_s)-\tilde\L(w_b)$ climbed
monotonically across an interval on which $c_\mathrm{min}\le\sigma^2_\mathrm{eff}\le c_\mathrm{max}$,
$c_\mathrm{min}>0$, has effective action
$U_\mathrm{eff}(w_s)-U_\mathrm{eff}(w_b)\in[\alpha/c_\mathrm{max},\alpha/c_\mathrm{min}]$,
and exactly $\alpha/\sigma^2_\mathrm{eff}$ on an interval of constant diffusion.
\textup{(vi)} Scale the fluctuation and the remainder at a fixed batch law,
$\delta\tilde\btheta=t\,\bu$ and $\br_2=t\,\bm{w}_1+t^2\bm{w}_2$. Then
$\sigma^2_\mathrm{eff}(t)-\sigma_s^2\sigma^2_\mathrm{obj}$ is a polynomial in
$t$ with no constant term, of linear coefficient
\[
2\,\be_b^\top\bH_{\tilde\btheta}\Cov_\bX[\bu,\bg_0]\,\be_b
\;+\;2\,\be_b^\top\Cov_\bX[\bg_0,\bm{w}_1]\,\be_b ,
\]
and with the fluctuation term
$\be_b^\top\bH_{\tilde\btheta}\Cov_\bX[\bu]\bH_{\tilde\btheta}\be_b$ in its
quadratic coefficient. A latent weight independent of the batch has
$\bSigma_\xi=\bm 0$ and $\Cov_\bX[\bg_0,\bm w_1]=\bm 0$ as well, $\bm w_1$ being the
curvature-fluctuation part of the remainder, so its excess begins at second order. A positive linear
coefficient gives $\sigma^2_\mathrm{eff}>\sigma_s^2\sigma^2_\mathrm{obj}$ at
every sufficiently small $t>0$, and a negative one reverses it.
\end{lemma}

\begin{proof}
\emph{Step 1: the three-term split.} The chain of \cref{sec:mechanism},
$\bg=\bg_0+\bH_{\tilde\btheta}\delta\tilde\btheta+\br_2$, defines $\br_2$;
by the regularity of \cref{sec:mechanism} $\bg_0$, $\delta\tilde\btheta$, and
$\br_2$ are square integrable and $\bH_{\tilde\btheta}$ is symmetric. Bilinearity of the
covariance and $\Cov_\bX[\delta\tilde\btheta,\bg_0]=\E_\bX[\delta\tilde\btheta\,\bxi^\top]$
give
\begin{equation}\label{eq:pf-split}
\bSigma_\mathrm{slack}=\Cov_\bX[\delta\tilde\btheta,\bg]
=\bSigma_\xi+\bV\bH_{\tilde\btheta}^{\top}+\Cov_\bX[\delta\tilde\btheta,\br_2]
=\bSigma_\xi+\bV\bH_{\tilde\btheta}+\Cov_\bX[\delta\tilde\btheta,\br_2],
\end{equation}
which is \textup{(i)}; its trace splits $\sigma^2_\mathrm{Jac}$, with
$\tr(\bV\bH_{\tilde\btheta})\ge0$ for a positive semidefinite
$\bH_{\tilde\btheta}$, as for any product of two positive semidefinite
matrices.

\emph{Step 2: the update covariance, exactly.} Both methods are read at
the common mean latent weight of \cref{sec:mechanism}; the direct one has
$\delta\tilde\btheta\equiv\bm 0$ and update $\bg_0$. By bilinearity, for the
square integrable summands of the chain,
\begin{equation}\label{eq:pf-cov3}
\begin{aligned}
\Cov_\bX[\bg]
&=\Cov_\bX[\bg_0]+\bH_{\tilde\btheta}\bV\bH_{\tilde\btheta}+\Cov_\bX[\br_2]
+\bigl(\bH_{\tilde\btheta}\bSigma_\xi+\bSigma_\xi^\top\bH_{\tilde\btheta}\bigr)\\
&\quad+\bigl(\Cov_\bX[\bg_0,\br_2]+\Cov_\bX[\br_2,\bg_0]\bigr)
+\bigl(\bH_{\tilde\btheta}\Cov_\bX[\delta\tilde\btheta,\br_2]+\Cov_\bX[\br_2,\delta\tilde\btheta]\bH_{\tilde\btheta}\bigr),
\end{aligned}
\end{equation}
the symmetry of $\bH_{\tilde\btheta}$ turning
$\Cov_\bX[\bg_0,\bH_{\tilde\btheta}\delta\tilde\btheta]=\bSigma_\xi^\top\bH_{\tilde\btheta}^\top$
into $\bSigma_\xi^\top\bH_{\tilde\btheta}$. Contracting with $\be_b$, by
$\be_b^\top(\bm A+\bm A^\top)\be_b=2\,\be_b^\top\bm A\be_b$,
\begin{equation}\label{eq:pf-projdiff}
\sigma^2_\mathrm{eff}=\be_b^\top\Cov_\bX[\bg_0]\,\be_b
+2\,\be_b^\top\bH_{\tilde\btheta}\bSigma_\xi\be_b
+\be_b^\top\bH_{\tilde\btheta}\bV\bH_{\tilde\btheta}\be_b
+\varrho(\be_b),
\end{equation}
the six $\br_2$ matrices of \eqref{eq:pf-cov3} projecting to the three
scalars of $\varrho(\be_b)$. With the positivity-map slope $s$ explicit every
term carries $s^2$; the display is that identity at $s=1$, $\bg$ already
including $\psi'$. At $\delta\tilde\btheta\equiv\bm 0$ the coupling, fluctuation, and remainder terms are covariances with a constant and vanish, whatever
$\bH_{\tilde\btheta}$. For $\be_b$ a coordinate direction of the slack, as it is in \cref{sec:mechanism}, \eqref{eq:chain} gives
$\be_b^\top\bg_0=\psi'(\bar w)\,\be_b^\top\nabla_\btheta\L$ at the coordinate
$\bar w=\be_b^\top\bar{\tilde\btheta}$ of the mean latent weight, so under
\textup{(A2)} the direct value is
\begin{equation}\label{eq:pf-direct}
\be_b^\top\Cov_\bX[\bg_0]\,\be_b=\psi'(\bar w)^2\,\sigma^2_\mathrm{obj}
\in\bigl[(\sigma_s-\varepsilon)^2,(\sigma_s+\varepsilon)^2\bigr]\sigma^2_\mathrm{obj},
\end{equation}
written $\sigma_s^2\sigma^2_\mathrm{obj}$ within the bracket of \textup{(A2)}.
This proves \textup{(ii)} and the first clause of \textup{(iii)}.
Substituting \eqref{eq:pf-split} into the coupling term,
\begin{equation*}
\sigma^2_\mathrm{eff}=\sigma_s^2\sigma^2_\mathrm{obj}
+2\,\be_b^\top\bH_{\tilde\btheta}\bSigma_\mathrm{slack}\be_b
-\be_b^\top\bH_{\tilde\btheta}\bV\bH_{\tilde\btheta}\be_b
+2\,\be_b^\top\Cov_\bX[\bg_0,\br_2]\be_b+\be_b^\top\Cov_\bX[\br_2]\be_b,
\end{equation*}
the self-coupling $\bV\bH_{\tilde\btheta}$ inside $\bSigma_\mathrm{slack}$
counting the fluctuation term twice and the middle covariance of $\varrho$ once;
\eqref{eq:sigma-eff} is exact on $\bSigma_\xi$ alone.

\emph{Step 3: the sign of the excess, and the orders.} For symmetric
$\bH_{\tilde\btheta}$ and the covariance $\bV\succeq0$,
\begin{equation}\label{eq:pf-jitter}
\be_b^\top\bH_{\tilde\btheta}\bV\bH_{\tilde\btheta}\be_b
=\Var_\bX\bigl[\langle\bH_{\tilde\btheta}\be_b,\delta\tilde\btheta\rangle\bigr]\ge0,
\end{equation}
with no assumption on the coupling, and $c$ of \cref{thm:main}\textup{(i)}, the coupling term plus the remainder, read at $\be_b$ with $\|\be_b\|=1$, is
\begin{equation}\label{eq:pf-netcoupling}
c=2\,\be_b^\top\bH_{\tilde\btheta}\bSigma_\xi\be_b+\varrho(\be_b).
\end{equation}
Adding \eqref{eq:pf-jitter} and \eqref{eq:pf-netcoupling} to
\eqref{eq:pf-projdiff} with \eqref{eq:pf-direct},
\begin{equation*}
\sigma^2_\mathrm{eff}\;\ge\;\sigma_s^2\sigma^2_\mathrm{obj}+c,
\end{equation*}
which is \textup{(iii)}, and strictly so wherever $c>0$; a one-dimensional
Gaussian instance with $\br_2\equiv\bm 0$ has $c\ge4$ with projected
diffusions $1$ and $9$. We assume the sign of $c$ wherever a
statement consumes it, for the reasons recorded under $c$ in \cref{sec:escape}. Scaling the fluctuation as $t\bu$
and the remainder as $t\bm{w}_1+t^2\bm{w}_2$ at a fixed batch law turns
\eqref{eq:pf-projdiff} into the exact polynomial
\begin{equation*}
\begin{aligned}
\sigma^2_\mathrm{eff}(t)-\sigma_s^2\sigma^2_\mathrm{obj}
&=t\bigl(2\,\be_b^\top\bH_{\tilde\btheta}\Cov_\bX[\bu,\bg_0]\be_b
+2\,\be_b^\top\Cov_\bX[\bg_0,\bm{w}_1]\be_b\bigr)\\
&\quad+t^2\bigl(\be_b^\top\bH_{\tilde\btheta}\Cov_\bX[\bu]\bH_{\tilde\btheta}\be_b+\cdots\bigr)
+t^3(\cdots)+t^4(\cdots),
\end{aligned}
\end{equation*}
each covariance being homogeneous of degree one in each argument, so that no
constant term survives and the fluctuation term, quadratic in the fluctuation,
cannot appear below degree two. A polynomial $a_1t+a_2t^2+a_3t^3+a_4t^4$
with $a_1\ne0$ has the sign of $a_1$ at every sufficiently small $t>0$, so a positive
linear coefficient gives the strict ordering there while a negative one
reverses it: at small scale the alignment decides and the fluctuation term, in the
quadratic coefficient, decides nothing. A latent weight independent of the
batch that forms the gradient has $\bSigma_\xi=\bm 0$, and the pairing
$\Cov_\bX[\bg_0,\bm{w}_1]$ of $\bg_0$ with the
curvature-fluctuation part of the remainder vanishes as well, so both
first-order terms are absent and the excess begins at second order. This is
\textup{(vi)}.

\emph{Step 4: the stationary density.} Projecting the It\^o process of
\textup{(A1)} onto $\be_b$, $w=\be_b^\top\tilde\btheta$, gives the scalar
equation of \textup{(iv)}, with Fokker--Planck diffusion coefficient
$D=(\eta/2)\sigma^2_\mathrm{eff}$. Its drift $-\tilde\L'(w)$ is the
projected mean update, each mean update being the gradient of its own
expected loss whenever an integrable bound dominates the update, and for
every finite batch law without one.
The passage to the stationary Fokker--Planck equation is the register of
\textup{(A1)}; what follows is proved. Continuity of $\tilde\L'$ and
continuity and positivity of $\sigma^2_\mathrm{eff}$ supply, by the
fundamental theorem of calculus, an antiderivative $U_\mathrm{eff}$ with
\begin{equation*}
U_\mathrm{eff}'=\frac{\tilde\L'}{\sigma^2_\mathrm{eff}},
\qquad
\pi=\frac{C}{\sigma^2_\mathrm{eff}}\exp\Bigl(-\frac{2}{\eta}U_\mathrm{eff}\Bigr),
\qquad
D\pi=\frac{C\eta}{2}\exp\Bigl(-\frac{2}{\eta}U_\mathrm{eff}\Bigr),
\end{equation*}
the prefactor $\sigma^{-2}_\mathrm{eff}$ canceling against $D$ exactly, so
that for $\eta\ne0$
\begin{equation}\label{eq:pf-zeroflux}
(D\pi)'=-C\,U_\mathrm{eff}'\exp\Bigl(-\frac{2}{\eta}U_\mathrm{eff}\Bigr)=-\tilde\L'\,\pi
\end{equation}
at every point: $\pi$ has zero probability current, with no expansion in
$\eta$. Every zero-flux density $\rho$ has this form, \eqref{eq:pf-zeroflux}
for $\rho$ reading $(D\rho\,e^{(2/\eta)U_\mathrm{eff}})'=0$, so
\begin{equation*}
\rho=\frac{C_\rho}{\sigma^2_\mathrm{eff}}\exp\Bigl(-\frac{2}{\eta}U_\mathrm{eff}\Bigr),
\qquad
C_\rho=\sigma^2_\mathrm{eff}(0)\,\rho(0)\,e^{(2/\eta)U_\mathrm{eff}(0)},
\end{equation*}
which licenses the definite article. The zero-flux condition follows from the
equation itself. A non-negative solution of the full stationary equation has a
constant current $J$, $(D\rho)'=-\tilde\L'\rho-J$, and
\begin{equation*}
D\rho\;=\;\Bigl(K-J\int_0^{w}e^{(2/\eta)U_\mathrm{eff}(u)}\,du\Bigr)\,e^{-(2/\eta)U_\mathrm{eff}(w)},
\end{equation*}
where for $\eta>0$ and $U_\mathrm{eff}$ bounded below the integral tends to
$\pm\infty$ as $w\to\pm\infty$, so $J\ne0$ drives $D\rho$ negative at one
end of the line; hence $J=0$. The lower bound is what a confining loss
supplies and cannot be dropped, a potential unbounded below admitting a
transient solution of nonzero current. On an interval of constant diffusion
$\sigma^2_\mathrm{eff}\equiv c$, $U_\mathrm{eff}=\tilde\L/c$ and
$\pi\propto\exp(-2\tilde\L/(\eta c))$: the Boltzmann measure is the
constant-diffusion case, not a borrowing.

\emph{Step 5: the barrier.} Let $\alpha=\tilde\L(w_s)-\tilde\L(w_b)>0$. On
an interval of constant diffusion $\sigma^2_\mathrm{eff}\equiv c$ the fundamental theorem of
calculus for $U_\mathrm{eff}'=\tilde\L'/c$ gives
\begin{equation*}
U_\mathrm{eff}(w_s)-U_\mathrm{eff}(w_b)
=\frac1c\int_{w_b}^{w_s}\tilde\L'(u)\,du=\frac{\alpha}{c},
\end{equation*}
exactly and with no monotonicity asked. On an interval with
$c_\mathrm{min}\le\sigma^2_\mathrm{eff}\le c_\mathrm{max}$ climbed
monotonically, so that $\tilde\L'\ge0$ on $[w_b,w_s]$,
\begin{equation*}
\frac{\tilde\L'}{c_\mathrm{max}}\le\frac{\tilde\L'}{\sigma^2_\mathrm{eff}}\le\frac{\tilde\L'}{c_\mathrm{min}}
\quad\text{pointwise, hence}\quad
\frac{\alpha}{c_\mathrm{max}}\le U_\mathrm{eff}(w_s)-U_\mathrm{eff}(w_b)\le\frac{\alpha}{c_\mathrm{min}},
\end{equation*}
the interval integrability of $\tilde\L'$ and $\tilde\L'/\sigma^2_\mathrm{eff}$
following from their continuity on the interval and the antiderivative from
continuity and positivity on an open interval containing it. The bounds pinch
to $\alpha/c$ on an interval of constant diffusion and can be strict on a non-constant one. Monotone climbing is an added restriction over $\alpha>0$:
it forbids an intervening critical point, the single-barrier profile of
\textup{(A2)} supplies it where \cref{cor:fpt} needs it, and without it the
pointwise inequalities reverse where $\tilde\L'<0$. Multiplying by $2/\eta$
gives the exponent $2\alpha/(\eta c)$ of \cref{cor:fpt} and its bracket
$[2\alpha/(\eta c_\mathrm{max}),\,2\alpha/(\eta c_\mathrm{min})]$.
\end{proof}

\begin{remark}[A perturbation independent of the batch]\label{rem:matched-cell}
A perturbation of the latent weight sampled independently of the batch
that forms the gradient reproduces $\bV$ and with it the fluctuation term, while
that independence sends the coupling to zero and the first-order covariance
$\Cov_\bX[\bg_0,\bm w_1]$ with it, and a centered latent weight independent of the batch has no coupling term. Its excess over the
direct value therefore begins at second order in the perturbation scale,
while at a positive linear coefficient the \lift{}'s begins at first
order, by \textup{(vi)} and Step 3. The ordering of the lemma survives in that case in its
non-strict form, $\sigma^2_\mathrm{eff}$ at least the direct value, provided
the remainder is dominated by the fluctuation term it perturbs; no strict ordering is
claimed there. The two parameter-fluctuation terms are therefore kept
as separately named objects throughout.
\end{remark}

\subsection{Escape time}
\label{app:proof-fpt}

\begin{corollary}[Escape time; \cref{thm:main}\textup{(iii)} in full]\label{cor:fpt}
Let Assumptions~\ref{as:A1}--\ref{as:A2} hold. Let
$\alpha=\tilde\L(w_s)-\tilde\L(w_b)>0$ be the barrier height along $\be_b$,
read $\sigma^2_\mathrm{eff}$ as constant across the
barrier, through the bracket of \cref{lem:diffusion}\textup{(v)}, at the
direct value $\sigma_s^2\sigma^2_\mathrm{obj}$ for direct
softplus and at \eqref{eq:sigma-eff} for the lift,
and let $\mathsf{T}(x)$ be the mean
first-passage time of $dw=-\tilde\L'(w)\,dt+\sqrt{\eta}\,\sigma_\mathrm{eff}\,dB_t$
from a start $x<w_s$ on the interval $[a,b]$ of
Assumption~\ref{as:A2}, reflecting at $a$ and absorbing at $b$, under
Dynkin's formula for its generator at every step size and continuous up to the closed interval.
Then, for every $\delta>0$, there is
$L(\delta)>0$, depending on $\delta$, on $\tilde\L$, on $[a,b]$,
and on the start $x$, and not on $\eta\sigma^2_\mathrm{eff}$, with
$\bigl(2L(\delta)/(\eta\sigma^2_\mathrm{eff})\bigr)\,e^{2(\alpha-\delta)/(\eta\sigma^2_\mathrm{eff})}\le\mathsf{T}(x)\le\bigl(2(b-x)(b-a)/(\eta\sigma^2_\mathrm{eff})\bigr)\,e^{2\alpha/(\eta\sigma^2_\mathrm{eff})}$,
written $\mathsf{T}(x)\asymp\exp(2\alpha/(\eta\,\sigma^2_\mathrm{eff}))$, and
$\eta\sigma^2_\mathrm{eff}\log\mathsf{T}(x)\to2\alpha$ as $\eta\sigma^2_\mathrm{eff}\to0^{+}$.
The exponent $2\alpha/(\eta\sigma^2_\mathrm{eff})$ is strictly decreasing in
$\sigma^2_\mathrm{eff}$, so whenever $c>0$ the \lift{}'s exponent is strictly
the smaller of the two, by \cref{lem:diffusion}\textup{(iii)},
and, under the same condition, for every sufficiently small step size $\eta$ the \lift{}'s mean
first-passage time is itself the smaller. Both orderings hold at the level of
the exponent.
\end{corollary}

\begin{proof}
\emph{Step 1: the reduction to $\be_b$.} Projected onto $\be_b$, the
It\^o process of \textup{(A1)} has state $w=\be_b^\top\tilde\btheta$, drift
$-\tilde\L'(w)$ with $\tilde\L$ read along $\be_b$, and diffusion coefficient
$\eta\sigma^2_\mathrm{eff}$: by Step 2 of \cref{app:proof-diffusion},
$\eta\sigma_s^2\sigma^2_\mathrm{obj}$ for the direct method, whose
$\delta\tilde\btheta\equiv\bm 0$, and $\eta$ times \eqref{eq:sigma-eff} for
the lift. The reduction is a lower bound on the multi-dimensional action and
not an identity: for a path $\bm\gamma=\bm\gamma_0+q\,\be_b$ confined to the
line, with $q$ differentiable, drift $\bm\mu(\bm\gamma)=m(q)\,\be_b$, and
diffusion matrix $\bSigma(\bm\gamma)\succ0$, the two
Freidlin--Wentzell~\citep{freidlin2012random}
integrands being interval integrable on $[0,\mathsf t]$,
\begin{equation}\label{eq:pf-fw}
\int_0^{\mathsf t}\frac{(\dot q-m(q))^2}{\be_b^\top\bSigma(\bm\gamma)\be_b}\,dt
\;\le\;
\int_0^{\mathsf t}(\dot{\bm\gamma}-\bm\mu(\bm\gamma))^\top
\bSigma(\bm\gamma)^{-1}(\dot{\bm\gamma}-\bm\mu(\bm\gamma))\,dt,
\end{equation}
by $(\be_b^\top\bSigma\be_b)^{-1}\le\be_b^\top\bSigma^{-1}\be_b$ for a unit
$\be_b$ (Cauchy--Schwarz), with equality exactly when $\be_b$ is an
eigendirection of $\bSigma$. The scalar understates the escape cost off an
eigendirection, and the comparison below is stated on the scalar.

\emph{Step 2: the boundary-value problem and the bracket.} Write
$\sigma^2=\eta\sigma^2_\mathrm{eff}$, held at its value on the interval, the constant
reading that \cref{lem:diffusion}\textup{(v)} brackets. Dynkin's formula for
the generator $\tfrac12\sigma^2\partial^2-\tilde\L'\partial$ is the
hypothesis that the mean exit time solves
\begin{equation}\label{eq:pf-bvp}
\tfrac12\sigma^2\,\mathsf{T}''(x)-\tilde\L'(x)\,\mathsf{T}'(x)=-1\ \text{ on }(a,b),
\qquad \mathsf{T}(b)=0,\qquad \mathsf{T}'(a^{+})=0,
\end{equation}
the modeling step of the corollary; everything downstream of it is proved.
Under continuity of $\tilde\L$ on $[a,b]$ and differentiability inside, two
differentiations through the fundamental theorem of calculus show that
\begin{equation}\label{eq:pf-closed}
\mathsf{T}(x)=\frac{2}{\sigma^2}\int_{x}^{b}e^{2\tilde\L(y)/\sigma^2}
\int_{a}^{y}e^{-2\tilde\L(z)/\sigma^2}\,dz\,dy
\end{equation}
solves \eqref{eq:pf-bvp}, and it is the only solution continuous on $[a,b]$:
for two solutions $\mathsf{T},\mathsf{V}$,
$\bigl((\mathsf{T}'-\mathsf{V}')e^{-2\tilde\L/\sigma^2}\bigr)'=0$ on $(a,b)$
with limit $0$ at $a^{+}$, so $\mathsf{T}'=\mathsf{V}'$ and
$\mathsf{T}-\mathsf{V}\equiv\mathsf{T}(b)-\mathsf{V}(b)=0$. Every mean exit
time given through Dynkin's formula is therefore \eqref{eq:pf-closed}. Fix
$\delta>0$. Continuity of $\tilde\L$ at $w_s$ and $w_b$ and the ordering
$a<w_b<w_s<b$, $a\le x<w_s$, give windows $[s_1,s_2]\subset[x,b]$ about
$w_s$ with $\tilde\L\ge\tilde\L(w_s)-\delta/2$ and $[t_1,t_2]\subset[a,s_1]$
about $w_b$ with $\tilde\L\le\tilde\L(w_b)+\delta/2$; the integrands of
\eqref{eq:pf-closed} being positive,
\begin{equation}\label{eq:pf-lower}
\mathsf{T}(x)\ge\frac{2}{\sigma^2}\int_{s_1}^{s_2}e^{2\tilde\L(y)/\sigma^2}
\int_{t_1}^{t_2}e^{-2\tilde\L(z)/\sigma^2}\,dz\,dy
\ge\frac{2L(\delta)}{\sigma^2}\,e^{2(\alpha-\delta)/\sigma^2},
\end{equation}
where $L(\delta)=(s_2-s_1)(t_2-t_1)>0$ depends on $\delta$, $\tilde\L$, $[a,b]$, and $x$, not on $\sigma^2$, and degenerates as
$\delta\to0$. Assumption~\ref{as:A2}\textup{(b)} makes $\tilde\L(w_s)$ and $\tilde\L(w_b)$ the
extreme values of $\tilde\L$ on $[a,b]$, so
\begin{equation}\label{eq:pf-upper}
\mathsf{T}(x)\le\frac{2}{\sigma^2}\int_{x}^{b}e^{2\tilde\L(w_s)/\sigma^2}
\int_{a}^{y}e^{-2\tilde\L(w_b)/\sigma^2}\,dz\,dy
\le\frac{2(b-x)(b-a)}{\sigma^2}\,e^{2\alpha/\sigma^2}.
\end{equation}
The pair \eqref{eq:pf-lower}--\eqref{eq:pf-upper} is what
$\mathsf{T}(x)\asymp\exp(2\alpha/(\eta\sigma^2_\mathrm{eff}))$ abbreviates;
at fixed diffusion it does not close, $L(\delta)\to0$, and at the level of
exponents it does:
\begin{equation*}
2(\alpha-\delta)+\sigma^2\log\frac{2L(\delta)}{\sigma^2}
\;\le\;\sigma^2\log\mathsf{T}(x)\;\le\;
2\alpha+\sigma^2\log\frac{2(b-x)(b-a)}{\sigma^2},
\end{equation*}
with $\sigma^2\log(K/\sigma^2)\to0$ as $\sigma^2\to0^{+}$ and $\delta$
arbitrary, so $\sigma^2\log\mathsf{T}(x)\to2\alpha$, for
\eqref{eq:pf-closed} and hence for any mean exit time given through Dynkin's
formula at every diffusion. The sharp Kramers prefactor is not addressed: a
function $Ke^{2\alpha/\sigma^2}$ satisfies both bounds on a window of small
diffusions, so the prefactor is settled by nothing here. Where the diffusion
varies across the interval, the exponent lies in
$[2\alpha/(\eta c_\mathrm{max}),\,2\alpha/(\eta c_\mathrm{min})]$ by
\cref{lem:diffusion}\textup{(v)} and equals $2\alpha/(\eta\sigma^2_\mathrm{eff})$
where it is constant.

\emph{Step 3: the comparison.} For $\alpha,\eta>0$ the map
$x\mapsto2\alpha/(\eta x)$ is strictly decreasing on $(0,\infty)$. With
$\sigma^2_\mathrm{dir}=\sigma_s^2\sigma^2_\mathrm{obj}$ and
$\sigma^2_\mathrm{lift}$ for \eqref{eq:sigma-eff},
\cref{lem:diffusion}\textup{(iii)} at $c>0$ gives
\begin{equation}\label{eq:pf-order}
0<\sigma^2_\mathrm{dir}<\sigma^2_\mathrm{dir}+c\le\sigma^2_\mathrm{lift}
\quad\Longrightarrow\quad
\frac{2\alpha}{\eta\sigma^2_\mathrm{lift}}<\frac{2\alpha}{\eta\sigma^2_\mathrm{dir}},
\end{equation}
$\sigma^2_\mathrm{dir}>0$ by $\sigma^2_\mathrm{obj}>0$ in \cref{sec:mechanism}
and $\sigma_s>\varepsilon$ in \textup{(A2)}; where the pointwise ordering
holds on the interval no bracket on either diffusion is needed. For the times themselves choose
$\delta<\alpha(1-\sigma^2_\mathrm{dir}/\sigma^2_\mathrm{lift})$, so that
$2(\alpha-\delta)/\sigma^2_\mathrm{dir}>2\alpha/\sigma^2_\mathrm{lift}$;
then \eqref{eq:pf-upper} at $\eta\sigma^2_\mathrm{lift}$ and
\eqref{eq:pf-lower} at $\eta\sigma^2_\mathrm{dir}$ give
\begin{equation*}
\mathsf{T}_\mathrm{lift}(x)
\le\frac{2(b-x)(b-a)}{\eta\sigma^2_\mathrm{lift}}\,e^{2\alpha/(\eta\sigma^2_\mathrm{lift})}
<\frac{2L(\delta)}{\eta\sigma^2_\mathrm{dir}}\,e^{2(\alpha-\delta)/(\eta\sigma^2_\mathrm{dir})}
\le\mathsf{T}_\mathrm{dir}(x)
\end{equation*}
for every sufficiently small $\eta$, with no gap hypothesis and none on
prefactors; under Dynkin's formula for both processes at every $\eta$ the
same holds for the mean first-passage times, by the uniqueness above. The exponent limit at each
diffusion gives
$\eta\log\bigl(\mathsf{T}_\mathrm{dir}(x)/\mathsf{T}_\mathrm{lift}(x)\bigr)\to2\alpha(1/\sigma^2_\mathrm{dir}-1/\sigma^2_\mathrm{lift})$,
and at a fixed $\eta$ the ratio is pinned by
\begin{equation*}
\frac{L(\delta)\,\sigma^2_\mathrm{lift}}{(b-x)(b-a)\,\sigma^2_\mathrm{dir}}\,
e^{\frac{2(\alpha-\delta)}{\eta\sigma^2_\mathrm{dir}}-\frac{2\alpha}{\eta\sigma^2_\mathrm{lift}}}
\;\le\;\frac{\mathsf{T}_\mathrm{dir}(x)}{\mathsf{T}_\mathrm{lift}(x)}\;\le\;
\frac{(b-x)(b-a)\,\sigma^2_\mathrm{lift}}{L(\delta)\,\sigma^2_\mathrm{dir}}\,
e^{\frac{2\alpha}{\eta\sigma^2_\mathrm{dir}}-\frac{2(\alpha-\delta)}{\eta\sigma^2_\mathrm{lift}}},
\end{equation*}
whose width is the missing prefactor. The action $\alpha$ belongs to the one
expected loss fixed across the comparison, as the common mean latent weight requires, so the lift enters through $\sigma^2_\mathrm{eff}$ alone.
\end{proof}

\begin{remark}[The power of the positivity-map slope]\label{rem:power}
On the shoulder the coupling changes the constant in the exponent and not its
order. \directcap{}'s diffusion there is $\sigma_s^2\sigma^2_\mathrm{obj}$, so
its exponent is of order $\sigma_s^{-2}$, and a diffusion
$\sigma_s^2\sigma^2_\mathrm{obj}+\sigma_s q$ with $q$ bounded away from zero on the interval, the form a term bearing a single factor of $\psi'$ would supply,
would give an exponent of order $\sigma_s^{-1}$ and a divergent ratio of the
two. The coupling term does not take that form. By the chain rule
$\bH_{\tilde\btheta}=\mathrm{diag}(\psi')\bH_\theta\mathrm{diag}(\psi')+\mathrm{diag}(\psi''\odot\bg_\theta)$
and $\bSigma_\xi=\bm\Xi_\theta\mathrm{diag}(\psi')$, with $\bH_\theta$ and
$\bg_\theta$ the mean Hessian and the mean gradient in constrained units and
$\bm\Xi_\theta=\E_\bX[\delta\tilde\btheta\,\bxi_\theta^\top]$ the coupling in
those units. For softplus, at a mean latent weight with $\psi'=\sigma_s$ in every
coordinate,
\begin{equation*}
2\,\be_b^\top\bH_{\tilde\btheta}\bSigma_\xi\be_b
\;=\;2\sigma_s^{3}\,\be_b^\top\bH_\theta\bm\Xi_\theta\be_b
\;+\;2\sigma_s^{2}(1-\sigma_s)\,\be_b^\top\mathrm{diag}(\bg_\theta)\bm\Xi_\theta\be_b,
\end{equation*}
which is $O(\sigma_s^2)$ and never of order $\sigma_s$. The same bound holds at
the relative tolerance $\rho=\varepsilon/\sigma_s$ of \textup{(A2)}, since
softplus has $\psi''\le\psi'$. The fluctuation term and the remainder carry
the same power, the latter because $\br_2=\psi'\odot\tilde\br$. Where the remainder's constant is below the leading coefficient of
\eqref{eq:sigma-eff}, both exponents are therefore of order $\sigma_s^{-2}$ and their ratio is
bounded, at the factor $\sigma^2_\mathrm{eff}/(\sigma_s^2\sigma^2_\mathrm{obj})$
of \cref{lem:diffusion}.
\end{remark}

\begin{corollary}[Time below the barrier top; the second half of \cref{thm:main}\textup{(iii)}]\label{cor:dwell}
Under the hypotheses of \cref{cor:fpt}, and under Dynkin's formula for its
generator at the indicator of $[a,w_s]$ at every step size, with the occupation time continuous up to the
closed interval and continuously differentiable on $(a,b)$, let $\tau$ be the first
passage of $w$ to $b$ and let
$\mathsf{D}(x)=\E_x\bigl[\int_0^{\tau}\mathbf 1_{[a,w_s]}(w_t)\,dt\bigr]$ be the
mean time the diffusion spends below the barrier top before escape. Then
$\eta\sigma^2_\mathrm{eff}\log\mathsf{D}(x)\to2\alpha$ and
$\eta\sigma^2_\mathrm{eff}\log\bigl(\mathsf{T}(x)/\mathsf{D}(x)\bigr)\to0$ as
$\eta\sigma^2_\mathrm{eff}\to0^{+}$. This time therefore carries the exponent of
\cref{cor:fpt} and is strictly the shorter under the \lift{} whenever $c>0$, for
every sufficiently small step size, while the fraction $\mathsf{D}(x)/\mathsf{T}(x)$ of the time before escape
spent below the barrier top is subexponential in
$\eta\sigma^2_\mathrm{eff}$.
\end{corollary}

\begin{proof}
Write $\sigma^2=\eta\sigma^2_\mathrm{eff}$. Dynkin's formula for the generator
$\tfrac12\sigma^2\partial^2-\tilde\L'\partial$, the same modeling step that
\eqref{eq:pf-bvp} is, states that $\mathsf{D}$ solves
$\tfrac12\sigma^2\mathsf{D}''-\tilde\L'\mathsf{D}'=-\mathbf 1_{[a,w_s]}$ at
every point of $(a,b)$ other than $w_s$, with $\mathsf{D}(b)=0$ and
$\mathsf{D}'(a^{+})=0$. At $w_s$ the indicator jumps and $\mathsf{D}'$ has a
corner, so the second derivative does not exist there and the equation is
stated off that point; $\mathsf{D}'$ is continuous across it by the hypothesis of the
corollary, which is what the uniqueness argument consumes. The two differentiations and the uniqueness
argument of \cref{app:proof-fpt} otherwise apply verbatim to this
right-hand side, bounded and continuous off $w_s$, so
\begin{equation}\label{eq:pf-dwell}
\mathsf{D}(x)=\frac{2}{\sigma^2}\int_{x}^{b}e^{2\tilde\L(y)/\sigma^2}
\int_{a}^{y}\mathbf 1_{[a,w_s]}(z)\,e^{-2\tilde\L(z)/\sigma^2}\,dz\,dy,
\end{equation}
which is \eqref{eq:pf-closed} with the inner integrand restricted. The
integrands being non-negative, $\mathsf{D}(x)\le\mathsf{T}(x)$, so
\eqref{eq:pf-upper} bounds $\mathsf{D}(x)$ above. For the lower bound take the
windows of \eqref{eq:pf-lower}: there $[t_1,t_2]\subset[a,s_1]$ with
$s_1\le w_s$, so $\mathbf 1_{[a,w_s]}\equiv1$ on $[t_1,t_2]$ and the same
computation gives
$\mathsf{D}(x)\ge(2L(\delta)/\sigma^2)\,e^{2(\alpha-\delta)/\sigma^2}$ for
every $\delta>0$. That the inner window lies below the barrier top is the
ordering $a<w_b<w_s$ of \textup{(A2)}, and it is the only step of
\cref{app:proof-fpt} that the indicator touches. The two bounds are those of
\cref{cor:fpt}, so
$\sigma^2\log\mathsf{D}(x)\to2\alpha$ by its exponent argument, and
$\sigma^2\log(\mathsf{T}(x)/\mathsf{D}(x))=\sigma^2\log\mathsf{T}(x)-\sigma^2\log\mathsf{D}(x)\to0$.
The comparison is Step 3 of \cref{app:proof-fpt} with \eqref{eq:pf-dwell} in
place of \eqref{eq:pf-closed}, unchanged.
\end{proof}

\subsection{Proof of \texorpdfstring{\cref{thm:main}}{Theorem 1}}
\label{app:proof-main}

\begin{proof}
The mean latent weight, $\be_b$, $\sigma^2_\mathrm{obj}$, and
$\sigma^2_\mathrm{eff}$ are fixed as in \cref{lem:diffusion}. Parts
\textup{(i)} and \textup{(ii)} are
\cref{lem:diffusion}\textup{(ii)}--\textup{(iii)} and \textup{(vi)}, under that
lemma's hypotheses less its parts \textup{(iv)}--\textup{(v)}, which leaves the regularity of \cref{sec:mechanism} and \textup{(A2a)}. Part
\textup{(iii)} is \cref{cor:fpt} together with \cref{cor:dwell}, under their
hypotheses, its reading of $\sigma^2_\mathrm{eff}$ as constant across the
barrier bracketed by \cref{lem:diffusion}\textup{(v)}. The sign condition
$c>0$ enters no part of the theorem, only the ordering those corollaries carry.
\end{proof}

\subsection{Preconditioned steps}
\label{app:precond}

\begin{proposition}[Preconditioned steps; \cref{prop:precond} in full]\label{prop:precond-full}
Let the hypotheses of \cref{cor:fpt,cor:dwell} hold with
Assumption~\ref{as:A1} replaced by the following. The lifted parameters are
trained by
$\bm\varphi_{k+1}=\bm\varphi_k-\eta\,\bP_k\nabla_{\bm\varphi}\L(\bm\varphi_k;\bX_k)$
on mini-batches that are independent and identically distributed, with each
$\bP_k$ diagonal with positive entries and determined by
$\bX_1,\ldots,\bX_{k-1}$, as Adam's normalization is when its second-moment
estimate is lagged by one step, and the recursion is read as the It\^o process
$d\bm{\varphi}=-\bP\nabla_{\bm{\varphi}}\tilde\L(\bm{\varphi})\,dt+\sqrt{\eta}\,\bP\bSigma^{1/2}(\bm{\varphi})\,d\bm{B}_t$
with the regularity of Assumption~\ref{as:A1}. Let $\be_b$ be an eigendirection
of $\bP$, $\bP\be_b=p\,\be_b$, as it is for a coordinate of the slack, and read
$p>0$ as constant across the barrier at its value at the mean latent weight. Then the following hold.

\textup{(i)} \Cref{lem:diffusion}\textup{(i)}--\textup{(iii)}, \textup{(v)}
and \textup{(vi)}, and with them \cref{thm:main}\textup{(i)}--\textup{(ii)}, and
\cref{thm:joint-necessity} with its two substitutions hold unchanged, and
\cref{lem:diffusion}\textup{(iv)} holds with $\eta p$ in place of $\eta$.

\textup{(ii)} Let $\mathsf{T}_p$ and $\mathsf{D}_p$ be the mean first passage
and the mean time below the barrier top of
$dw=-p\,\tilde\L'(w)\,dt+\sqrt{\eta}\,p\,\sigma_\mathrm{eff}\,dB_t$, under the
boundary conditions and the reading of Dynkin's formula of
\cref{cor:fpt,cor:dwell}, and let $\mathsf{T}^{(\eta')}$ and
$\mathsf{D}^{(\eta')}$ be those of \cref{cor:fpt,cor:dwell} at step $\eta'$.
Then $\mathsf{T}_p=\mathsf{T}^{(\eta p)}/p$ and
$\mathsf{D}_p=\mathsf{D}^{(\eta p)}/p$, so at fixed $p$,
$\eta p\sigma^2_\mathrm{eff}\log\mathsf{T}_p(x)\to2\alpha$,
$\eta p\sigma^2_\mathrm{eff}\log\mathsf{D}_p(x)\to2\alpha$ and
$\eta p\sigma^2_\mathrm{eff}\log\bigl(\mathsf{T}_p(x)/\mathsf{D}_p(x)\bigr)\to0$
as $\eta\sigma^2_\mathrm{eff}\to0^{+}$.

\textup{(iii)} Let the $p$ of each method be $\chi(m^2+\sigma^2_\mathrm{eff})$ at
its own $\sigma^2_\mathrm{eff}$, with $m$ the mean gradient along $\be_b$ at the mean latent weight, shared by the two methods as $\alpha$ is, and
$\chi:(0,\infty)\to(0,\infty)$ non-increasing with $v\mapsto v\,\chi(m^2+v)$
strictly increasing on $(0,\infty)$. Then
$\Psi(v)=2\alpha/(\eta\,v\,\chi(m^2+v))$ is strictly decreasing on
$(0,\infty)$. With $\sigma^2_\mathrm{dir}=\sigma_s^2\sigma^2_\mathrm{obj}$,
$\sigma^2_\mathrm{lift}$ the value \eqref{eq:sigma-eff}, and
$\kappa=\sigma^2_\mathrm{lift}/\sigma^2_\mathrm{dir}>1$, which $c>0$ gives,
\begin{equation}\label{eq:precond-ratio}
1\;<\;\frac{\Psi(\sigma^2_\mathrm{dir})}{\Psi(\sigma^2_\mathrm{lift})}
\;=\;\kappa\,\frac{\chi(m^2+\sigma^2_\mathrm{lift})}{\chi(m^2+\sigma^2_\mathrm{dir})}
\;\le\;\kappa,
\end{equation}
and for every sufficiently small $\eta$ the \lift{}'s mean first passage and
mean time below the barrier top are the smaller. The ratio is at least
$\kappa^{1-\gamma}$ for $\chi(v)=(v+\epsilon)^{-\gamma}$ with $\epsilon\ge0$ and
$\gamma\in[0,1)$, or $\gamma=1$ with $m^2+\epsilon>0$, and at least
$\kappa^{1/2}$ for Adam's normalization $\chi(v)=(v^{1/2}+\epsilon)^{-1}$ with
$\epsilon\ge0$.
\end{proposition}

\begin{proof}
\emph{Step 1: what does not move.} Parts \textup{(i)}--\textup{(iii)} and
\textup{(vi)} of \cref{lem:diffusion} and \cref{thm:joint-necessity} are
statements about the law of the gradient and of the fluctuation over batches at
one iterate, with the parameters fixed, and part \textup{(v)} is one about
$U_\mathrm{eff}$ alone, so the update rule enters none of their proofs;
\cref{thm:joint-necessity} uses \textup{(A1)} only for its independent and
identically distributed batches, which the preconditioned recursion keeps. Part
\textup{(iv)} projects the It\^o process, whose drift $-p\tilde\L'$ and
diffusion coefficient $\eta p^2\sigma^2_\mathrm{eff}$ make
$\sigma^{-2}_\mathrm{eff}\exp(-(2/(\eta p))U_\mathrm{eff})$ the zero-flux density
by Step 4 of \cref{app:proof-diffusion}. Part \textup{(i)} follows.

\emph{Step 2: the rescaling.} Step 1 of \cref{app:proof-fpt} applies with the
diffusion matrix $\bP\bSigma\bP$ in place of $\bSigma$, and since
$\bP\be_b=p\,\be_b$ the projection onto $\be_b$ has drift $-p\,\tilde\L'(w)$
and diffusion coefficient $\eta p^2\sigma^2_\mathrm{eff}$. Its generator is
\begin{equation*}
\tfrac12\,\eta p^2\sigma^2_\mathrm{eff}\,\partial^2-p\,\tilde\L'\partial
\;=\;p\,\bigl(\tfrac12\,(\eta p)\,\sigma^2_\mathrm{eff}\,\partial^2-\tilde\L'\partial\bigr),
\end{equation*}
$p$ times the generator of \cref{cor:fpt} at step $\eta p$. The
boundary-value problem \eqref{eq:pf-bvp} for $\mathsf{T}_p$ therefore states
that $p\,\mathsf{T}_p$ solves \eqref{eq:pf-bvp} at step $\eta p$, with the same
boundary conditions, and the uniqueness in Step 2 of \cref{app:proof-fpt} gives
$\mathsf{T}_p=\mathsf{T}^{(\eta p)}/p$; the same computation with the indicator
right-hand side of \cref{cor:dwell} gives $\mathsf{D}_p=\mathsf{D}^{(\eta p)}/p$.
At fixed $p$, $\eta p\sigma^2_\mathrm{eff}\log(1/p)\to0$, so the three limits
are those of \cref{cor:fpt,cor:dwell} at diffusion $\eta p\sigma^2_\mathrm{eff}$.

\emph{Step 3: the comparison.} By Step 2 the exponent of each method is $\Psi$
at its own $\sigma^2_\mathrm{eff}$, and $\Psi$ is strictly decreasing because
$v\,\chi(m^2+v)$ is strictly increasing. The middle term of
\eqref{eq:precond-ratio} is
$\sigma^2_\mathrm{lift}\chi(m^2+\sigma^2_\mathrm{lift})/\bigl(\sigma^2_\mathrm{dir}\chi(m^2+\sigma^2_\mathrm{dir})\bigr)$,
which exceeds one by the same monotonicity and is at most $\kappa$ because $\chi$
is non-increasing. For the times, Step 3 of \cref{app:proof-fpt} applies with
$\eta p_\mathrm{lift}\sigma^2_\mathrm{lift}$ and
$\eta p_\mathrm{dir}\sigma^2_\mathrm{dir}$ in place of the two diffusions and
with the factors $1/p$ of Step 2 in the prefactors. Choosing
$\delta<\alpha\bigl(1-p_\mathrm{dir}\sigma^2_\mathrm{dir}/(p_\mathrm{lift}\sigma^2_\mathrm{lift})\bigr)$
makes the exponential gap grow without bound as $\eta\to0^{+}$, while the
prefactors are polynomial in $1/\eta$. For the power family,
\begin{equation*}
\frac{d}{dv}\,v\,(m^2+v+\epsilon)^{-\gamma}
=(m^2+v+\epsilon)^{-\gamma-1}\bigl((1-\gamma)v+m^2+\epsilon\bigr)>0
\end{equation*}
under the stated conditions, and
$\chi(m^2+\sigma^2_\mathrm{lift})/\chi(m^2+\sigma^2_\mathrm{dir})
=\bigl((m^2+\sigma^2_\mathrm{dir}+\epsilon)/(m^2+\sigma^2_\mathrm{lift}+\epsilon)\bigr)^{\gamma}\ge\kappa^{-\gamma}$,
since $(A+u)/(B+u)\ge A/B$ for $0<A\le B$ and $u\ge0$. For Adam's
normalization, with $s=(m^2+v)^{1/2}$,
\begin{equation*}
\frac{d}{dv}\,\frac{v}{s+\epsilon}=\frac{2m^2+v+2\epsilon s}{2s\,(s+\epsilon)^2}>0,
\end{equation*}
and
$\chi(m^2+\sigma^2_\mathrm{lift})/\chi(m^2+\sigma^2_\mathrm{dir})
=(s_\mathrm{dir}+\epsilon)/(s_\mathrm{lift}+\epsilon)\ge s_\mathrm{dir}/s_\mathrm{lift}\ge\kappa^{-1/2}$
by the same inequality.
\end{proof}

\begin{remark}[The power of the slope under a preconditioned step]\label{rem:power-precond}
Under the reading of \cref{rem:power}, with $\alpha$ fixed and
$\psi'=\sigma_s$ at the mean latent weight, the mean gradient and the noise along
$\be_b$ carry one factor of $\sigma_s$ each, so
$m^2+\sigma^2_\mathrm{dir}=\sigma_s^2\mu^2$ with
$\mu^2=\bar g_\theta^2+\sigma^2_\mathrm{obj}$ the second moment in constrained
units. For $\chi(v)=(v+\epsilon)^{-\gamma}$ the \direct{} exponent is then
$2\alpha(\sigma_s^2\mu^2+\epsilon)^{\gamma}/(\eta\sigma_s^2\sigma^2_\mathrm{obj})$,
as $\sigma_s\to0$ of order $\sigma_s^{-2(1-\gamma)}$ at $\epsilon=0$
and of order $\sigma_s^{-2}$ at a fixed $\epsilon>0$, the law of
\cref{cor:fpt} at step $\eta\epsilon^{-\gamma}$. Adam's normalization gives the
same two regimes, of order $\sigma_s^{-1}$ at $\epsilon=0$ and of order
$\sigma_s^{-2}$ at a fixed $\epsilon>0$, the law of \cref{cor:fpt} at step
$\eta\epsilon^{-1}$. The \lift{}'s exponent has the same orders where its
remainder meets the condition of \cref{rem:power}, and the
ratio of the two is bounded by \eqref{eq:precond-ratio}. A step of this family with $\gamma>0$, or Adam's normalization,
therefore lowers the order of both exponents alike at $\epsilon=0$, and the coupling still changes their constant and not
their order.
\end{remark}

Adam's averaging of the first moment, its momentum, lies outside \cref{prop:precond}, and extending the escape analysis to steps with momentum is future work.

\subsection{Projection}
\label{app:proof-pgd}

This subsection states the results behind the comparison with the projection.
Projection outputs raw parameters into the loss, so its
update includes no factor of $\psi'$, and its constrained weights are stored
variables that no batch conditions. We read its gradient noise into the diffusion undiminished, $\bSigma_\mathrm{PGD}=\bSigma_\mathrm{obj}$. This
identification is an input to every statement below, and it leaves out what
the return step does to the update on the active set. With the positivity-map slope kept explicit, the three diffusion
matrices at a common mean latent weight are
$\bSigma_\mathrm{PGD}=\bSigma_\mathrm{obj}$,
$\bSigma_\mathrm{dir}=\sigma_s^2\bSigma_\mathrm{obj}$, and
$\bSigma_\mathrm{lift}=\sigma_s^2(\bSigma_\mathrm{obj}+\bm{C}+\bH_{\tilde\btheta}\bV\bH_{\tilde\btheta})$
with $\bm{C}=\bH_{\tilde\btheta}\bSigma_\xi+\bSigma_\xi^\top\bH_{\tilde\btheta}$
the coupling matrix, and each is read along a unit escape
direction $\bv$ as the quadratic form $\bv^\top\bSigma\bv$. This is the one
place where the conversion of \cref{app:proofs}'s opening applies: the
excess terms of \eqref{eq:sigma-eff} are $\sigma_s^2$ times
$\bv^\top(\bm C+\bH_{\tilde\btheta}\bV\bH_{\tilde\btheta})\bv$, and the comparison is stated in the paper's convention.
\cref{thm:directional} is stated at a vanishing remainder; the threshold reads the excess net of the remainder that $c$ folds into itself, which is the same
arithmetic with the remainder added to the excess. The common mean latent weight of \cref{sec:mechanism} is not itself licensed
for
a scheme with no latent coordinate, and the comparison assumes it in
addition.

No account resting on the size of the exploration alone therefore settles the
\lift{} against projection. The \textcolor{liftdirect}{direct} method is
the coldest of the three where the coupling term is non-negative along
$\bv$, since $0<\sigma_s<1$ puts it below projection and
$\bv^\top\bH_{\tilde\btheta}\bV\bH_{\tilde\btheta}\bv\ge0$ puts it at or
below the \lift{}. The \lift{} is above or below projection along $\bv$
according to the threshold of \cref{thm:directional}, in either direction.

\begin{theorem}[Directional escape comparison]\label{thm:directional}
Let $\alpha>0$, $\eta>0$, and $0<\sigma_s$, and let $\bv$ be a direction on
which both directional diffusions are positive,
$\bv^\top\bSigma_\mathrm{PGD}\bv>0$ and $\bv^\top\bSigma_\mathrm{lift}\bv>0$.
Then the lift's Arrhenius exponent $2\alpha/(\eta\,\bv^\top\bSigma_\mathrm{lift}\bv)$
is strictly below projection's $2\alpha/(\eta\,\bv^\top\bSigma_\mathrm{PGD}\bv)$
if and only if
\begin{equation*}
(\sigma_s^{-2}-1)\,\bv^\top\bSigma_\mathrm{obj}\bv \;<\;
\bv^\top\bigl(\bm C+\bH_{\tilde\btheta}\bV\bH_{\tilde\btheta}\bigr)\bv,
\end{equation*}
equivalently, in the paper's convention, if and only if the coupling and
fluctuation terms of \eqref{eq:sigma-eff} along $\bv$ exceed $\sigma_s^{-2}-1$
times the \direct{} value $\sigma_s^2\,\bv^\top\bSigma_\mathrm{obj}\bv$.
If in addition \textup{(A2b)} holds along $\bv$ with Dynkin's formula for
both mean exit times at every projection diffusion, each continuous up to the
closed interval, and the inequality holds, then the lift's mean
first-passage time is strictly below projection's for every sufficiently
small projection diffusion, the lifted diffusion exceeding it by the surplus
that the inequality fixes.
\end{theorem}

The threshold is out of reach in the regime the paper is about. Written in
absolute terms the bar is $(1-\sigma_s^2)\,\bv^\top\bSigma_\mathrm{obj}\bv$, which
rises to $\bv^\top\bSigma_\mathrm{obj}\bv$ as the positivity map attenuates, while the
coupling and fluctuation terms of~\eqref{eq:sigma-eff} are of order $\sigma_s^2$ by
the power count of \cref{rem:power} and fall to zero with it. The
inequality therefore fails for every sufficiently small $\sigma_s$, and the
directional test can be met only where the positivity-map derivative is of order one,
away from the shoulder. What the threshold decides is consequently whether
projection is overtaken outside the attenuated region, and not whether the
mechanism of \cref{sec:diffusion} favors the \lift{} within it.

\begin{proof}
For $\alpha,\eta>0$ the map $x\mapsto2\alpha/(\eta x)$ is strictly decreasing
on $(0,\infty)$, so with both directional diffusions positive
\begin{equation*}
\begin{aligned}
\frac{2\alpha}{\eta\,\bv^\top\bSigma_\mathrm{lift}\bv}<\frac{2\alpha}{\eta\,\bv^\top\bSigma_\mathrm{PGD}\bv}
&\iff
\bv^\top\bSigma_\mathrm{obj}\bv<\sigma_s^2\,\bv^\top\bigl(\bSigma_\mathrm{obj}+\bm C+\bH_{\tilde\btheta}\bV\bH_{\tilde\btheta}\bigr)\bv\\
&\iff
(\sigma_s^{-2}-1)\,\bv^\top\bSigma_\mathrm{obj}\bv<\bv^\top\bigl(\bm C+\bH_{\tilde\btheta}\bV\bH_{\tilde\btheta}\bigr)\bv,
\end{aligned}
\end{equation*}
the last step dividing by $\sigma_s^2>0$. Dividing by
$\bv^\top\bSigma_\mathrm{obj}\bv>0$ states the same condition on the ratio of the excess to $\bv^\top\bSigma_\mathrm{obj}\bv$ compared with $\sigma_s^{-2}-1$, and
multiplying both sides by $\sigma_s^2$ gives the paper's convention. For the
exit times let $\sigma^2_\mathrm{PGD}=\eta\,\bv^\top\bSigma_\mathrm{PGD}\bv$
and $\Delta=\eta\,\bv^\top(\bSigma_\mathrm{lift}-\bSigma_\mathrm{PGD})\bv>0$,
the surplus the inequality fixes. Under \textup{(A2b)}
along $\bv$ and Dynkin's formula for both, the bracket
\eqref{eq:pf-lower}--\eqref{eq:pf-upper} at the diffusions
$\sigma^2_\mathrm{PGD}$ and $\sigma^2_\mathrm{PGD}+\Delta$, with a fixed
$\delta\in(0,\alpha)$, gives
\begin{equation*}
\mathsf{T}_\mathrm{lift}(x)
\le\frac{2(b-x)(b-a)}{\sigma^2_\mathrm{PGD}+\Delta}\,e^{2\alpha/(\sigma^2_\mathrm{PGD}+\Delta)}
<\frac{2L(\delta)}{\sigma^2_\mathrm{PGD}}\,e^{2(\alpha-\delta)/\sigma^2_\mathrm{PGD}}
\le\mathsf{T}_\mathrm{PGD}(x)
\end{equation*}
for every sufficiently small $\sigma^2_\mathrm{PGD}$, the middle exponent
diverging while the first tends to $2\alpha/\Delta$; at a prescribed pair of
diffusions the ordering holds under the explicit gap inequality between the
two middle terms.
\end{proof}

The reduction of the multi-dimensional action to the scalar along $\bv$ is
that of Step 1 of \cref{app:proof-fpt}, and the hypothesis that the improving
move is a crossing along one fixed direction is inherited from it. The
theorem is a biconditional and it is true whatever a measurement of its
right-hand side returns; it decides what the paper may claim, not whether the
comparison holds. The threshold is a real discriminator: a positively aligned
configuration clears it at $\sigma_s=3/5$ and fails it at $\sigma_s=1/3$, so positive alignment alone
is not sufficient.

\begin{theorem}[Projection has no coupling term]\label{thm:pgd-zero}
Let the latent weight not vary with the conditioning batch, as it does not for
projection and for direct softplus. Then for every $T\ge2$ the
cross-covariance estimator of \eqref{eq:cross-cov} vanishes, the coupling $\bSigma_\xi$ vanishes with the fluctuation and the coupling matrix
$\bm C$ with it, the lift's diffusion matrix evaluated at
slope one, zero coupling, and zero fluctuation term equals $\bSigma_\mathrm{PGD}$,
the ratio of the excess to $\bv^\top\bSigma_\mathrm{obj}\bv$ along every direction is
zero, and the threshold of \cref{thm:directional} fails at slope one and at
every slope below one.
\end{theorem}

\begin{proof}
A latent weight that does not vary with the batch has
$\delta\tilde\btheta_k=\tilde\btheta_k-\bar{\tilde\btheta}=\bm 0$ on every
probe and $\delta\tilde\btheta\equiv\bm 0$ in population, so for every
$T\ge2$
\begin{gather*}
\widehat\bSigma^{(t)}=\bm 0\ \text{ summand by summand in \eqref{eq:cross-cov}},
\qquad
\bSigma_\xi=\E_\bX[\delta\tilde\btheta\,\bxi^\top]=\bm 0,
\\[2pt]
\bm C=\bH_{\tilde\btheta}\bSigma_\xi+\bSigma_\xi^\top\bH_{\tilde\btheta}=\bm 0 .
\end{gather*}
Substituting into the definitions at $\sigma_s=1$, $\bm C=\bm 0$, and
$\bH_{\tilde\btheta}\bV\bH_{\tilde\btheta}=\bm 0$,
\begin{equation*}
\bSigma_\mathrm{lift}=\bSigma_\mathrm{obj}=\bSigma_\mathrm{PGD},
\qquad
\frac{\bv^\top(\bm C+\bH_{\tilde\btheta}\bV\bH_{\tilde\btheta})\bv}{\bv^\top\bSigma_\mathrm{obj}\bv}=0
\quad\text{for every }\bv,
\end{equation*}
and the threshold of \cref{thm:directional}, $\sigma_s^{-2}-1<0$ in ratio
form, fails at $\sigma_s=1$, where it reads $0<0$, and at every
$\sigma_s<1$, where $\sigma_s^{-2}-1>0$.
\end{proof}

The theorem does not close the converse of \cref{thm:joint-necessity}: a
structural zero does not by itself say anything about the loss. It says that
the alignment route of \cref{thm:directional} is unavailable to
projection, whose exploration is whatever direction the minibatch gradient
noise supplies and which no amount of temperature redirects, whereas the
\lift{}'s coupling term is built from the same batch that produces the gradient
fluctuation. That leaves the threshold of \cref{thm:directional} as a
prediction with an explicit test. Contract the curvature-weighted coupling with the escape direction, $\sigma_s^2\,\bv^\top\bm C\bv$, which is
the coupling term
$2\,\bv^\top\bH_{\tilde\btheta}\bSigma_\xi\bv$ of \eqref{eq:sigma-eff} along
$\bv$, and compare it against $\sigma_s^{-2}-1$ times the \direct{} value along $\bv$. Clearing that bar settles the threshold, since the
fluctuation term the contraction omits is non-negative, while failing it leaves
the question open, since an instance with zero coupling and a positive fluctuation term clears the threshold.
The trace of \eqref{eq:cross-cov} is the wrong monitor for this test, being
blind to direction. The test is directional and
concerns the coupling alone, we have not run it, and we assert nothing
about whether the threshold holds at our operating point.

\begin{proposition}[One projected coordinate at the boundary; the Lindley recursion]\label{prop:clamp}
Let $\gamma,\sigma,\eta>0$, let $\zeta_0,\zeta_1,\ldots$ be independent standard
normal with distribution function $\Phi$, and let
$W_{k+1}=\max\bigl(0,\,W_k-\eta(\gamma+\sigma\zeta_k)\bigr)$, $W_0=0$, with
$\rho=\gamma/\sigma$. The visit to the boundary begun at $W_0=0$ ends at the first step whose update
points into the cone, so it lasts a geometric number of steps of mean
$1/\Phi(-\rho)$, and so does every later visit, whatever the step at which it begins. At every step the clamped fraction $\mathbb{P}(W_k=0)$ is free of
$\eta$ and $\sigma$, a function of $\rho$ alone, non-decreasing in $\rho$, and
tends to $1$ as $\rho\to\infty$. Writing
$S_j=-\eta\sigma\sum_{i<j}(\rho+\zeta_i)$, we have $W_k=\max_{0\le j\le k}S_j$ in
distribution; $M=\sup_{j\ge0}S_j$ is finite almost surely with finite mean
$\E[M]=\eta\sigma\,m(\rho)$, $m$ depending on $\rho$ alone; and $W_k\to M$ in
distribution. The law of $M$ is the recursion's unique stationary law, under which
$\mathbb{P}(W_\infty=0)=\exp\bigl(-\sum_{k\ge1}k^{-1}\Phi(-\rho\sqrt{k})\bigr)$
\citep{spitzer1956}, decreasing to $0$ as $\rho\downarrow0$.
\end{proposition}

\begin{proof}
Write $X_k=-\eta(\gamma+\sigma\zeta_k)=-\eta\sigma(\rho+\zeta_k)$, so that
$W_{k+1}=(W_k+X_k)^{+}$ with $X_0,X_1,\ldots$ independent and normal of mean
$-\eta\gamma<0$ and variance $\eta^2\sigma^2$, and $S_j=X_0+\cdots+X_{j-1}$.

\emph{The visit length.} From $W_k=0$,
\begin{equation}\label{eq:pf-release}
W_{k+1}>0\iff X_k>0\iff \zeta_k<-\rho,\qquad
\mathbb{P}(\zeta_k<-\rho)=\Phi(-\rho)\in(0,1).
\end{equation}
The $\zeta_k$ are independent, so the clamped steps before the release are
geometric of that success probability, and the visit, those steps together with
the one that ends it, has mean $1/\Phi(-\rho)$, at every step size and at every
depth. Since $W_k$ depends on $\zeta_0,\ldots,\zeta_{k-1}$ alone, the same law
holds for a visit begun at any later step, given the step at which it begins.

\emph{The clamped fraction at every step.} Dividing the recursion by $\eta\sigma$
gives $W_{k+1}/(\eta\sigma)=\max\bigl(0,\,W_k/(\eta\sigma)-\rho-\zeta_k\bigr)$, so
at a fixed step $W_k/(\eta\sigma)$ is a function of $\rho$ and
$\zeta_0,\ldots,\zeta_{k-1}$ alone, and the event $W_k=0$ is free of $\eta$ and
$\sigma$. By induction on $k$, $W_k$ is non-increasing in $\rho$ on every path,
so $\mathbb{P}(W_k=0)$ is non-decreasing in $\rho$. A coordinate that receives
no inward update in its first $k$ steps stays at zero, so
$\mathbb{P}(W_k=0)\ge\mathbb{P}(\zeta_i\ge-\rho,\ i<k)=(1-\Phi(-\rho))^k\to1$ as
$\rho\to\infty$.

\emph{The stationary law.} Iterating the recursion from $W_0=0$ gives
$W_k=\max_{0\le j\le k}S_j$ in distribution, non-decreasing in $k$. Since $S_j$
is normal of mean $-j\eta\gamma$ and variance $j\eta^2\sigma^2$,
$\mathbb{P}(S_j>0)=\Phi(-\rho\sqrt j)\le e^{-\rho^2j/2}$, free of $\eta$ and
summable in $j$, so almost surely only finitely many $S_j$ are positive,
$M=\sup_{j\ge0}S_j$ is finite, and $W_k\to M$ in distribution from every
starting point \citep{asmussen2003applied}. The law of $M$ is stationary:
$M=\max\bigl(0,\,X_0+M'\bigr)$ with $M'=\sup_{j\ge0}(S_{j+1}-S_1)$ independent of $X_0$ and
distributed as $M$. It is the only stationary law, because the recursion started from any
stationary law converges to it. The ladder-epoch identity
$\mathbb{P}(M=0)=\exp\bigl(-\sum_{j\ge1}j^{-1}\mathbb{P}(S_j>0)\bigr)$
\citep{spitzer1956} gives the stated form. Each term decreases in $\rho$, so the
clamped fraction increases in $\rho$; it tends to $1$ as $\rho\to\infty$ by
dominated convergence, and to $0$ as $\rho\downarrow0$ because
$\sum_{j\le\rho^{-2}}j^{-1}\Phi(-\rho\sqrt j)\ge\Phi(-1)\sum_{j\le\rho^{-2}}j^{-1}$
diverges.

\emph{The offset.} The recursion is positively homogeneous in $\eta\sigma$:
$M=\eta\sigma\sup_{j\ge0}\bigl(-\sum_{i<j}(\rho+\zeta_i)\bigr)$, whose
supremum depends on $\rho$ alone and has a finite mean, since $\E[M]$ is at
most $\sum_{j\ge1}\E[S_j^{+}]$, which the same Gaussian tail makes finite
\citep{asmussen2003applied}. Hence
$\E[W_\infty]=\eta\sigma\,m(\rho)$: the stationary iterate sits above the
constrained optimum at a distance proportional to the step size.
\end{proof}

What projection pays at the boundary is the re-crossing and not the stay, which
is all the proposition licenses.
 
\section{Experimental details}
\label{app:details}

Every run trains with Adam~\citep{kingma2015adam} and reports the iterate of its lowest held-out
loss, except the two-dimensional flows, which report their final iterate, and
the image regularizer, which reports the iterate of its highest validation
PSNR. The
tabular benchmarks take the splits and preprocessing of
\citet{papamakarios2017maf}, the two-dimensional flows the published recipe of
\citet{huang2021cpflow}, and the image regularizer that of
\citet{ehrhardt2026primaldual}. The energy-based models need the log-normalizer
of their target, which is estimated by importance sampling with the trained
sampler, a HINT flow~\citep{kruse2021hint}, as proposal, and by quadrature in one dimension. On the image
regularizer the \lift{} scales its permutation-invariant summary to a fixed
norm, which we found crucial for stable training.

\begin{table}[!ht]
\centering
\caption{Configuration of each experiment: the data dimension, the ICNN width and depth, the hypernetwork widths, the batch size, the step size, the iteration count,
and the seeds each comparison reports.}
\label{tab:config}
\small
\setlength{\tabcolsep}{5pt}
\begin{tabular}{lrllrrrr}
\toprule
experiment & $D$ & ICNN & hypernetwork & $n$ & $\eta$ & iterations & seeds \\
\midrule
one-dimensional targets & 1 & $128\times3$ & 64, 64, 96 & 256 & $10^{-3}$ & 3{,}000 & 30 \\
POWER & 6 & $256\times3$ & 96, 96, 128 & 512 & $10^{-3}$ & 4{,}000 & 3 \\
HEPMASS & 21 & $512\times5$ & 64, 64, 96 & 256 & $10^{-3}$ & 4{,}000 & 3 \\
MiniBooNE & 43 & $512\times5$ & 64, 64, 96 & 256 & $10^{-3}$ & 4{,}000 & 3 \\
two-dimensional flows & 2 & $5\times(32\times3)$ & 64, 64, 96 & 128 & $5\!\times\!10^{-3}$ & 19{,}550 & 5 \\
Darcy transport map & 100 & $128\times4$ & 128, 128, 192 & 256 & $10^{-3}$ & 4{,}000 & 10 \\
image regularizer & $256^2$ & 32 filters & 32, 64, 128 & 8 & $5\!\times\!10^{-4}$ & 2{,}500 & 24 \\
\bottomrule
\end{tabular}
\end{table}

\begin{table}[!ht]
\centering
\caption{The configuration each method is reported at on the tabular
datasets, where a ramp raises the step on the energy, and on MiniBooNE on the
sampler as well.}
\label{tab:schedules}
\small
\begin{tabular}{llll}
\toprule
 & POWER & HEPMASS & MiniBooNE \\
\midrule
\lift{} & none & ramp & ramp and sampler \\
\pgd{} & ramp & none & ramp and sampler \\
\direct{} & none & none & none \\
\bottomrule
\end{tabular}
\end{table}

On the image regularizer we also train \direct{} and \pgd{} at two to twenty
times the published step size, and the \lift{} at half and twice it, on two seeds each and with its summary scaled to norm one instead of the one half of the reported runs
(\cref{fig:ffhq-steps}).

\begin{figure}[!ht]
\centering
\includegraphics[width=0.6\linewidth]{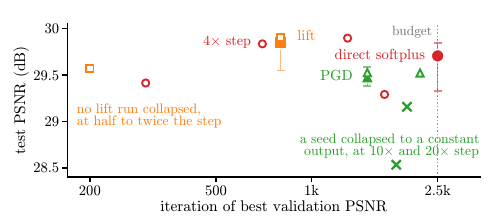}
\caption{Test PSNR at the iteration of best validation PSNR on the image
regularizer against that iteration, at the published step size and, as open
markers, at other step sizes; a cross marks a step size at which a seed
collapsed to a constant output.}
\label{fig:ffhq-steps}
\end{figure}
 \FloatBarrier

\clearpage
\section{The held-out read along training}
\label{app:val-loss}

\begin{figure}[!ht]
\centering
\includegraphics[width=\linewidth]{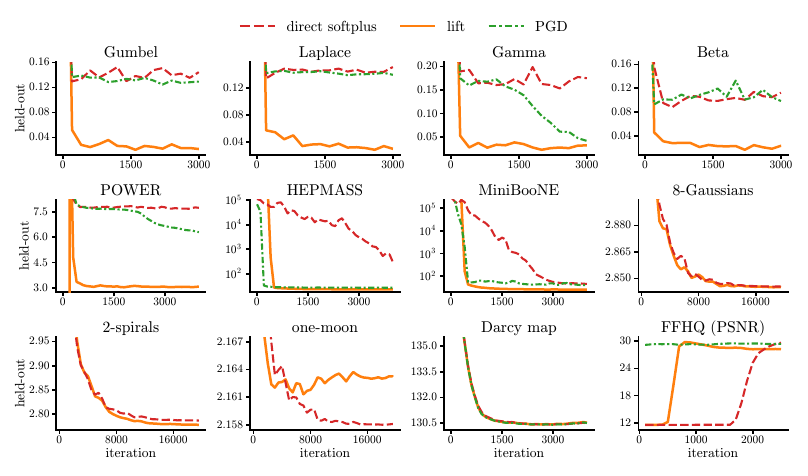}
\caption{The held-out read along training on each problem, the pointwise
median over seeds of each method. Higher is better on the image
regularizer and lower on the rest.}
\label{fig:val-loss}
\end{figure}
 \FloatBarrier

\section{One-dimensional gallery}
\label{app:gallery}

\begin{table}[!ht]
\centering
\caption{Total-variation distance to the target at the best held-out loss, the
median over seeds with the tenth and ninetieth percentiles. The \lift{} and
\direct{} bands are disjoint on every target.}
\label{tab:gallery}
\small
\begin{tabular}{lccc}
\toprule
target & \lift{} & \direct{} & \pgd{} \\
\midrule
Gumbel  & 0.016 [0.007, 0.029] & 0.137 [0.130, 0.142] & 0.130 [0.118, 0.144] \\
Laplace & 0.027 [0.015, 0.044] & 0.140 [0.138, 0.142] & 0.139 [0.136, 0.145] \\
Gamma   & 0.012 [0.006, 0.061] & 0.125 [0.113, 0.138] & 0.038 [0.024, 0.055] \\
Beta    & 0.012 [0.006, 0.025] & 0.073 [0.062, 0.087] & 0.072 [0.055, 0.090] \\
\bottomrule
\end{tabular}
\end{table}

The shoulder depends on the threshold $\sigma_\mathrm{sh}$ on the derivative of
the map, which is $0.05$ throughout. We also compute the share of weights on
the shoulder at the end of training and the rate at which a visit to it ends
at the thresholds $0.02$ and $0.10$ (\cref{fig:thresholds}). At each threshold
the \lift{} ends with more weights on the shoulder than \direct{}, and its
visits end far sooner on every target but Beta, where the rates of the two
methods are close.

\begin{figure}[!ht]
\centering
\includegraphics[width=\linewidth]{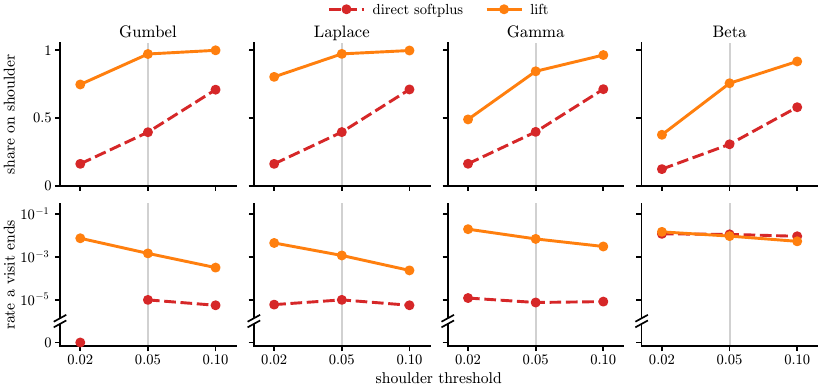}
\caption{The share of weights on the shoulder at the end of training (top) and
the rate at which a visit to it ends (bottom) at three thresholds on the
derivative of the map, medians over seeds, with the one used throughout in
gray. At each the \lift{} ends with more weights there, and its visits end
sooner except on Beta.}
\label{fig:thresholds}
\end{figure}

On Gumbel and Beta we also count how often weights enter and leave the
shoulder after the first snapshot interval, which contains the move away from
the initialization (\cref{fig:entry-exit}). On Gumbel a weight of \direct{} on the
shoulder barely leaves it, while weights of the \lift{} enter and leave far
more often. On Beta the two methods leave about as often, and weights of the
\lift{} enter far more often.

\begin{figure}[!ht]
\centering
\includegraphics[width=\linewidth]{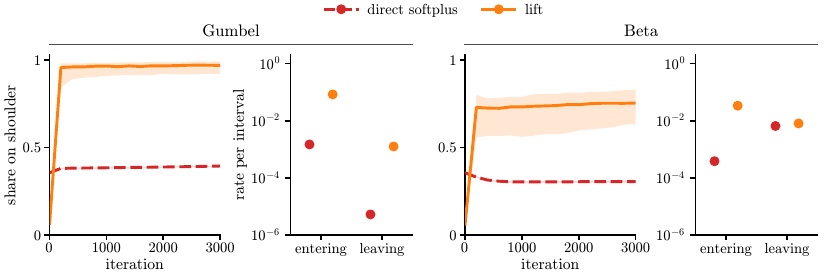}
\caption{The share of weights on the shoulder along training, the median over
seeds inside the band of the tenth to ninetieth percentiles, and the median
rates at which weights enter and leave the shoulder after the first snapshot
interval, on Gumbel and Beta.}
\label{fig:entry-exit}
\end{figure}
\FloatBarrier

\section{Classification on CIFAR-10}
\label{app:cifar}

\citet{sivaprasad2021curious} train input-convex classifiers and keep their
constrained weights non-negative by a rule applied after each step, which
replaces a negative weight $w$ by $\exp(w-5)$. This rule acts on the stored
weight, as \pgd{} does, so \cref{thm:directional,thm:pgd-zero} apply to it under the reading \cref{app:proof-pgd} takes for projection.
We compare it with the \lift{}, \direct{}, and \pgd{} on their CIFAR-10
task~\citep{krizhevsky2009cifar} and on the multilayer perceptron of their
paper, which has three hidden layers of 800 units and the ELU activation. Its
first layer is unconstrained and the layers after it are constrained. As in
the rest of the paper, the network has no batch normalization. We also leave
out their data augmentation and train for at most forty epochs, so our
accuracies are lower than theirs.

Each method trains with Adam on mini-batches of 64 images, and at a given seed
the constrained weights of all methods start from the same values under the
initialization of \citet{hoedt2023principled}, those of the \lift{} up to the
output of its hypernetwork. Training stops once the
accuracy on a validation split of ten percent of the training images has not
improved for eight epochs. The step size of a method is the one among
$10^{-4}$, $3\times10^{-4}$, and $10^{-3}$ that gives the best validation
accuracy on the first seed, the smaller one on a tie. We evaluate the test
accuracy at the epoch of best validation accuracy, over five seeds per
method. At both evaluations the hypernetwork of the \lift{} reads one fixed
batch of 64 training images.

The \lift{} has the highest median test accuracy of the four methods and is
above the rule on four of five seeds (\cref{fig:cifar}). For the rule the two smaller step sizes tie, and at the
smaller one it trains more slowly than the other three methods. A variant
with $\exp(w-8.8)$, whose value at zero is the median weight at
initialization, is selected at the step size of \pgd{} and trains as fast as
\pgd{}. It also ends below the \lift{} on four of five seeds. The training and
validation curves of the \lift{}, \direct{}, and \pgd{} overlap.

\begin{figure}[!ht]
\centering
\includegraphics[width=\linewidth]{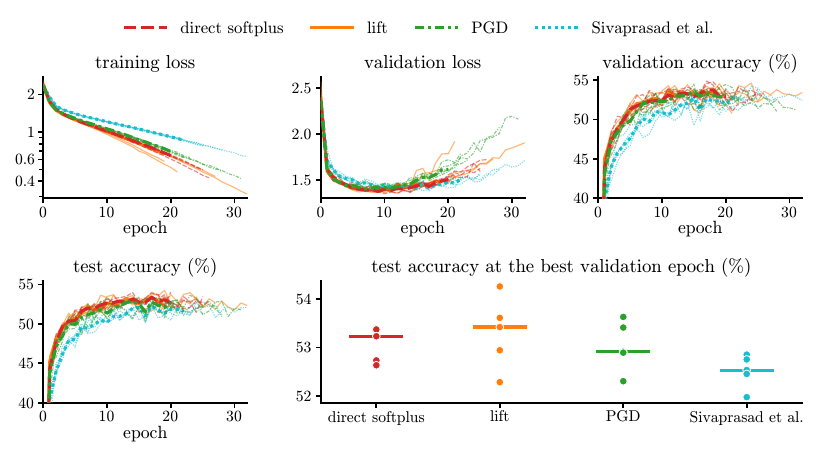}
\caption{Training and held-out curves on the CIFAR-10 task of
\citet{sivaprasad2021curious}, each method at its selected step size, seeds
thin and their median bold until the first seed stops. The bottom-right panel
gives the test accuracy of each seed at its best validation epoch and the
median as a bar.}
\label{fig:cifar}
\end{figure}
 \FloatBarrier

\section{The exponential map}
\label{app:expmap}

\Cref{thm:main,thm:joint-necessity} and \cref{prop:precond} ask of the
positivity map only that it be continuously differentiable and non-decreasing,
with a derivative close to a constant where the weights are
(Assumption~\ref{as:A2}). The exponential map $\psi(u)=e^{u}$~\citep{sill1998monotonic,hoedt2023principled}
meets these conditions as softplus does, and on the shoulder the two maps agree, since the
derivative of softplus is $e^{u}+O(e^{2u})$ for negative $u$. The derivative of
the exponential map equals the weight itself, so a weight is on its shoulder
exactly when it is at most $\sigma_\mathrm{sh}$. A weight there barely moves for the same reason as under
softplus, and the \lift{} acts on it through the same coupling term. The two
maps differ for large weights, where the exponential multiplies the gradient by
the weight itself and softplus by a factor close to one.

We train three methods on the four one-dimensional targets of
\cref{sec:exp-ebm} with the configuration of \cref{tab:config}: direct training
of the latent weight under the exponential map, the \lift{} under it, and the
\lift{} under softplus. The step size of each method on each target is selected
from $5\times10^{-4}$, $10^{-3}$, and $2\times10^{-3}$ on the held-out loss of
one seed. It is $10^{-3}$, the step size of the one-dimensional gallery (\cref{tab:config}), unless
another lowers the best held-out loss by $0.002$ nats or more. Five
further seeds then train at the selected step sizes. For direct training on
Gamma and Beta the selected step size is at an end of this range. The
\lift{} under softplus is trained again with the other two methods, so its
runs are not those of \cref{tab:gallery}, which reports thirty.

Under the exponential map every run of the \lift{} is closer to the target than
every run of direct training on the same target (\cref{fig:expmap}). Within the
first fifth of training, each of its runs also reaches a lower loss on a common held-out sample
than the best direct run does over all of training. Its median is close to that of the \lift{} under
softplus on Gumbel and Laplace, lower on Beta, and higher on Gamma. Direct training
under the exponential map stays at the level of \direct{}
(\cref{tab:gallery}). Visits to the shoulder behave as under softplus. At each
of the three thresholds (\cref{fig:thresholds}), a visit of the \lift{} ends
sooner than one of direct training on every seed of Gumbel, Laplace, and Gamma.
Beta is again the exception.

\begin{figure}[!ht]
\centering
\includegraphics[width=\linewidth]{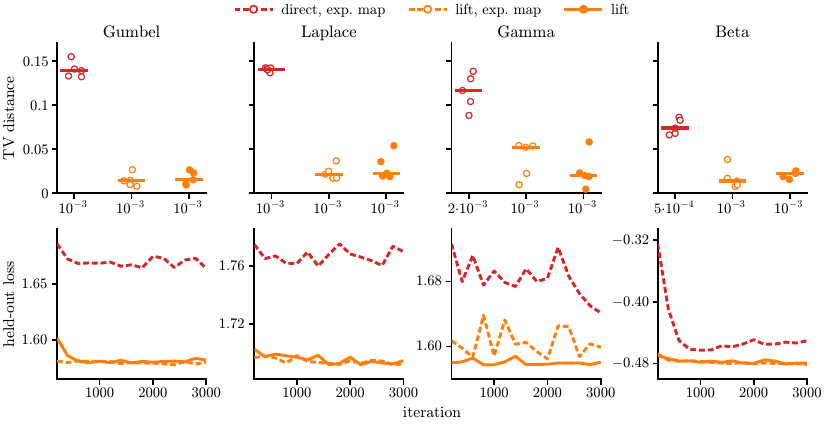}
\caption{Total-variation distance to the target at the best held-out loss under
the exponential map and under softplus, one point per seed with the median as a
bar and the selected step size below each method (top), and the held-out loss
along training (bottom). Every run of the \lift{} under the exponential map is
closer to the target than every run of direct training.}
\label{fig:expmap}
\end{figure}
\FloatBarrier

\section{Loss-landscape geometry}
\label{sec:exp-landscape}

\Cref{fig:landscape} renders the loss along two-dimensional slices through
the final iterates of single runs of a convex potential flow, three
energy-based models, and the transport map, in the manner
of \citet{li2018visualizing}. In the constrained
coordinates $\btheta$ the first direction runs from the final iterate of the
\lift{}, at the origin, to that of \direct{}. In the lifted coordinates
$(\bphi,\bb)$ both directions are principal directions of the \lift{}'s
trajectory. On the two larger energy-based models the constrained slice places
\direct{}'s final iterate above the \lift{}'s. While the constrained slice separates the two, the lifted slice is
a single valley that the \lift{}'s trajectory follows. On Darcy, where the
three methods tie (\cref{sec:exp-darcy}), the three final iterates
coincide when each is scored under its own map.
 
\begin{figure}[!ht]
\centering
\includegraphics[width=\linewidth]{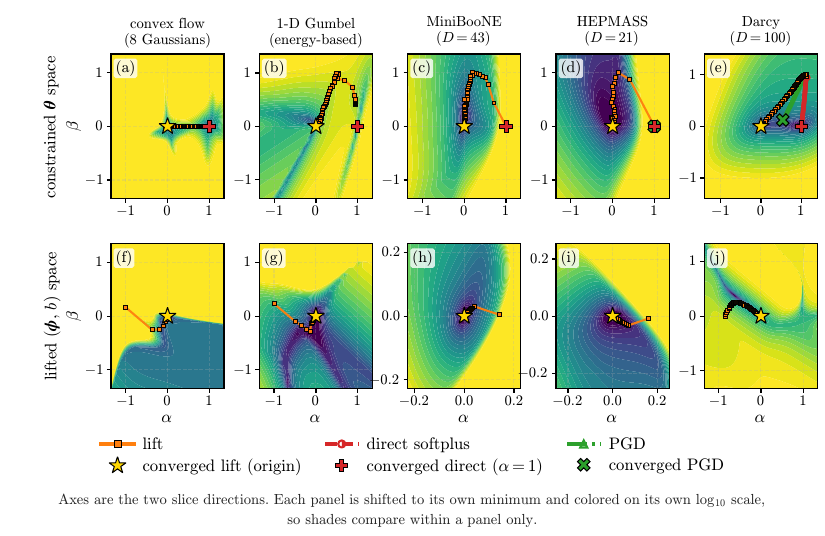}
\caption{Loss along two-dimensional slices through the final iterates, in the
constrained coordinates with the \lift{} at the origin and \direct{} at the
right (top) and in the lifted coordinates (bottom).}
\label{fig:landscape}
\end{figure}

\end{document}